\documentclass[11pt]{article}
\pdfoutput=1
\usepackage{enumerate}
\usepackage{pdfsync}
\usepackage[OT1]{fontenc}
\usepackage{smile}
\usepackage[colorlinks,
            linkcolor=red,
            anchorcolor=blue,
            citecolor=blue
            ]{hyperref}
\usepackage{color, colortbl}
\usepackage{fullpage}
\usepackage{hyperref}
\usepackage[protrusion=true,expansion=true]{microtype}
\usepackage{pbox}
\usepackage{setspace}
\usepackage{tabularx}
\usepackage{float}
\usepackage{wrapfig,lipsum}
\usepackage{enumitem}
\usepackage{microtype}
\usepackage{graphicx}
\usepackage{subfigure}
\usepackage{enumerate}
\usepackage{enumitem}
\usepackage{pgfplots}

\usetikzlibrary{arrows,shapes,snakes,automata,backgrounds,petri}
\usepackage{booktabs} 

\allowdisplaybreaks

\newcommand{\la}{\langle}
\newcommand{\ra}{\rangle}
\newcommand{\qvalue}{Q}
\newcommand{\vvalue}{V}

\newcommand{\reward}{r}

\newcommand{\event}{\mathcal{E}}
\newcommand{\ip}[1]{\langle #1 \rangle}

\newcommand{\seq}[1]{\overline{[#1]}}

\newcommand{\var}{\mathbb{V}}

\def \algbandit {\text{WeightedOFUL}^+}
\def \algmdp {\text{HF-UCRL-VTR}^+}
\def \algvar {\text{HOME}}
\def \error {E}
\newcommand{\state}{x}
\usepackage{enumitem}

\def \pnorm {B}

\usepackage{colortbl}
\definecolor{LightCyan}{rgb}{0.8, 0.9, 1}
\usepackage[utf8]{inputenc} 
\usepackage[T1]{fontenc}    
\usepackage{hyperref}       
\usepackage{url}            
\usepackage{booktabs}       
\usepackage{amsfonts}       
\usepackage{nicefrac}       
\usepackage{microtype}      
\usepackage{xcolor}         

\usepackage{enumitem}

\makeatletter
\newcommand*{\rom}[1]{\expandafter\@slowromancap\romannumeral #1@}
\makeatother
\title{\huge Computationally Efficient Horizon-Free Reinforcement Learning for Linear Mixture MDPs}

\author
{
	Dongruo Zhou\thanks{Department of Computer Science, University of California, Los Angeles, CA 90095, USA; e-mail: {\tt drzhou@cs.ucla.edu}} 
	~~~and~~~
	Quanquan Gu\thanks{Department of Computer Science, University of California, Los Angeles, CA 90095, USA; e-mail: {\tt qgu@cs.ucla.edu}}
}

\begin{document}
\date{}
\maketitle

\begin{abstract}
Recent studies have shown that episodic reinforcement learning (RL) is not more difficult than contextual bandits, even with a long planning horizon and unknown state transitions. However, these results are limited to either tabular Markov decision processes (MDPs) or computationally inefficient algorithms for linear mixture MDPs. In this paper, we propose the first computationally efficient horizon-free algorithm for linear mixture MDPs, which achieves the optimal $\tilde O(d\sqrt{K} +d^2)$ regret up to logarithmic factors. Our algorithm adapts a weighted least square estimator for the unknown transitional dynamic, where the weight is both \emph{variance-aware} and \emph{uncertainty-aware}. When applying our weighted least square estimator to heterogeneous linear bandits, we can obtain an $\tilde O(d\sqrt{\sum_{k=1}^K \sigma_k^2} +d)$ regret in the first $K$ rounds, where $d$ is the dimension of the context and $\sigma_k^2$ is the variance of the reward in the $k$-th round. This also improves upon the best-known algorithms in this setting when $\sigma_k^2$'s are known.
\end{abstract}

\section{Introduction}
How to design efficient algorithms is a central problem for reinforcement learning (RL). Here, the \emph{efficiency} includes both \emph{statistical efficiency}, which requires the RL algorithm enjoy a low regret/polynomial sample complexity for finding the near-optimal policy, and \emph{computational efficiency}, which expects the RL algorithm have polynomial running time. When restricting to episodic RL with total reward upper bounded by $1$\footnote{See Assumption \ref{ass:totalreward} for a detailed description.}, a longstanding question is whether episodic RL is statistically and computationally more difficult than contextual bandits \citep{jiang2018open}, since episodic RL can be seen as an extension of contextual bandits to have a long planning horizon and unknown state transition. For tabular RL, this questions has been fully resolved by a line of works \citep{wang2020long,zhang2021reinforcement,li2022settling, zhang2022horizon}, which propose various horizon-free algorithms. Here we say an algorithm is horizon-free if its regret/sample complexity has at most a polylogarithmic dependence on the planning horizon $H$. In particular, \citep{zhang2021reinforcement} proposed the first computationally efficient algorithm for tabular RL whose regret enjoys a polylogarithmic dependence on the planning horizon, and \citep{zhang2022horizon} further removed the polylogarithmic dependence on the planning horizon.  

For RL with \emph{function approximation} to deal with large state space, \citet{zhang2021improved,kim2021improved} have made some progress towards horizon-free RL for a class of MDPs called \emph{linear mixture MDPs} \citep{jia2020model, ayoub2020model, zhou2021nearly}, whose transition dynamic can be represented as a linear combination of $d$ basic transition models. More specifically, \citet{zhang2021improved} proposed a VARLin algorithm for linear mixture MDPs with an $\tilde O(d^{4.5}\sqrt{K} + d^9)$ regret for the first $K$ episodes, and \citet{kim2021improved} proposed a VARLin2 algorithm with an improved regret $\tilde O(d\sqrt{K} + d^2)$. However, neither algorithm is computationally efficient, because both of them need to work with nonconvex confidence sets and do not provide a polynomial-time algorithm to solve the maximization problem over these sets.

So the following question remains open:
\begin{center}
    \emph{Can we design computationally efficient horizon-free RL algorithms when function approximation is employed?}
\end{center}
In this paper, we answer the above question affirmatively for linear function approximation by proposing the first computationally efficient horizon-free RL algorithm for linear mixture MDPs. Our contributions are summarized as follows.
\begin{itemize}[leftmargin = *]
    \item As a warm-up, we consider the heterogeneous linear bandits where the variances of rewards in each round are different. Such a setting can be regarded as a special case of linear mixture MDPs. We propose a computationally efficient algorithm $\algbandit$ and prove that in the first $K$-rounds, the regret of $\algbandit$ is $\tilde O(d\sqrt{\sum_{k=1}^K \sigma_k^2} + dR + d)$, where $\sigma_k^2$ is the variance of the reward in the $k$-th round and $R$ is the magnitude of the reward noise. Our regret is \emph{variance-aware}, i.e., it only depends on the summation of variances and does not have a $\sqrt{K}$ term. This directly improves the $\tilde O(d\sqrt{\sum_{k=1}^K \sigma_k^2} + \sqrt{dK} + d)$ regret achieved by \citet{zhou2021nearly}.
    \item For linear mixture MDPs, when the total reward for each episode is upper bounded by $1$, we propose a $\algmdp$ algorithm and show that it has an $\tilde O(d\sqrt{K} + d^2)$ regret for the first $K$ episodes, where $d$ is the number of basis transition dynamic. Our $\algmdp$ is \emph{computationally efficient}, \emph{horizon-free} and \emph{near-optimal}, as it matches the regret lower bound proved in our paper up to logarithmic factors. Our regret is strictly better than the regret attained in previous works \citep{zhou2021nearly, zhang2021improved, kim2021improved}.
    \item At the core of both $\algbandit$ and $\algmdp$ is a carefully designed \emph{weighted linear regression} estimator, whose weight is both \emph{variance-aware} and \emph{uncertainty-aware}, in contrast to previous weighted linear regression estimator that is only \emph{variance-aware} \citep{zhou2021nearly}. For linear mixture MDPs, we further propose a $\algvar$ that constructs the variance-uncertainty-aware weights for high-order moments of the value function, which is pivotal to obtain a horizon-free regret for linear mixture MDPs. 
\end{itemize}

For a better comparison between our results and previous results, we summarize these results in Table \ref{table:11}. It is evident that our results improve upon all previous results in the respective settings\footnote{The only exception is that for heterogeneous linear bandits, our algorithm needs to know the noise variance, while \citet{zhang2021improved, kim2021improved} do not. But their algorithms are computationally inefficient. We plan to extend our algorithm to deal with unknown variance in the future work.}. 

\newcolumntype{g}{>{\columncolor{LightCyan}}c}
\begin{table*}[ht]
\caption{Comparisons of regrets for linear bandits and linear mixture MDPs.}\label{table:11}
\centering
\begin{tabular}{gggg}
\toprule
\rowcolor{white} & & & Computationally\\
\rowcolor{white} \multirow{-2}{*}{Algorithm}  & \multirow{-2}{*}{Regret} & \multirow{-2}{*}{Assumption} & Efficient?\\
\midrule
\rowcolor{white} OFUL & & &\\
\rowcolor{white} \tiny{\citep{abbasi2011improved}}  & \multirow{-2}{*}{$\tilde O(d\sqrt{K})$} & \multirow{-2}{*}{-} & \multirow{-2}{*}{Yes}\\
\rowcolor{white} WeightedOFUL & & &\\
\rowcolor{white} \small{\citep{zhou2021nearly}}  & \multirow{-2}{*}{$\tilde O(d\sqrt{\sum_{k=1}^K\sigma_k^2} + \sqrt{dK} + d)$} & \multirow{-2}{*}{Known variance} & \multirow{-2}{*}{Yes}\\
\rowcolor{white} VOFUL & & &\\
\rowcolor{white} \small{\citep{zhang2021improved}}  & \multirow{-2}{*}{$\tilde O(d^{4.5}\sqrt{\sum_{k=1}^K\sigma_k^2} + d^5)$} & \multirow{-2}{*}{Unknown variance} & \multirow{-2}{*}{No}\\
\rowcolor{white} VOFUL2 & & &\\
\rowcolor{white} \small{\citep{kim2021improved}}& \multirow{-2}{*}{$\tilde O(d^{1.5}\sqrt{\sum_{k=1}^K\sigma_k^2} + d^2)$} & \multirow{-2}{*}{Unknown variance}&\multirow{-2}{*}{No}\\
$\algbandit$ & & &\\
\small{(Theorem \ref{coro:linearregret})}& \multirow{-2}{*}{$\tilde O(d\sqrt{\sum_{k=1}^K\sigma_k^2} + d)$} & \multirow{-2}{*}{Known variance} &\multirow{-2}{*}{Yes}\\
\midrule
\rowcolor{white} UCRL-VTR & & Homogeneous,&\\
\rowcolor{white} \tiny{\citep{jia2020model,ayoub2020model}}  &\multirow{-2}{*}{$\tilde O(d\sqrt{H^3K})$} & $\sum_h r_h \leq H$ & \multirow{-2}{*}{Yes}\\
\rowcolor{white} UCRL-VTR$^+$ & $\tilde O(\sqrt{d^2H^3 + dH^4}\sqrt{K}$& Inhomogeneous,&\\
\rowcolor{white} \small{\citep{zhou2021nearly}}  & $+ d^2H^3 + d^3H^2)$ & $\sum_h r_h \leq H$ & \multirow{-2}{*}{Yes}\\
\rowcolor{white} VARLin & & Homogeneous,&\\
\rowcolor{white} \small{\citep{zhang2021improved}}  & \multirow{-2}{*}{$\tilde O(d^{4.5}\sqrt{K}+d^9)$} & $\sum_h r_h \leq 1$ & \multirow{-2}{*}{No}\\
\rowcolor{white} VARLin2 & & Homogeneous,&\\
\rowcolor{white} \small{\citep{kim2021improved}}  & \multirow{-2}{*}{$\tilde O(d\sqrt{K}+d^2)$} & $\sum_h r_h \leq 1$ & \multirow{-2}{*}{No}\\
$\algmdp$ & & Homogeneous,&\\
 \small{(Theorem \ref{thm:regret:finite})}  & \multirow{-2}{*}{$\tilde O(d\sqrt{K}+d^2)$} & $\sum_h r_h \leq 1$ & \multirow{-2}{*}{Yes}\\
 Lower bound & & &\\
 \small{(Theorem \ref{prop:lowerbound})}  & \multirow{-2}{*}{$\Omega(d\sqrt{K})$} & \multirow{-2}{*}{-} & \multirow{-2}{*}{-}\\
\bottomrule
\end{tabular}
\end{table*}

\paragraph{Notation} 
We use lower case letters to denote scalars, and use lower and upper case bold face letters to denote vectors and matrices respectively. We denote by $[n]$ the set $\{1,\dots, n\}$, by $\seq{n}$ the set $\{0,\dots, n-1\}$. For a vector $\xb\in \RR^d$ and a positive semi-definite matrix $\bSigma\in \RR^{d\times d}$, we denote by $\|\xb\|_2$ the vector's Euclidean norm and define $\|\xb\|_{\bSigma}=\sqrt{\xb^\top\bSigma\xb}$. For $\xb, \yb\in \RR^d$, let $\xb\odot\yb$ be the Hadamard (componentwise) product of $\xb$ and $\yb$. For two positive sequences $\{a_n\}$ and $\{b_n\}$ with $n=1,2,\dots$, 
we write $a_n=O(b_n)$ if there exists an absolute constant $C>0$ such that $a_n\leq Cb_n$ holds for all $n\ge 1$ and write $a_n=\Omega(b_n)$ if there exists an absolute constant $C>0$ such that $a_n\geq Cb_n$ holds for all $n\ge 1$. We use $\tilde O(\cdot)$ to further hide the polylogarithmic factors. We use $\ind\{\cdot\}$ to denote the indicator function. For $a,b \in \RR$ satisfying $a \leq b$, we use 
$[x]_{[a,b]}$ to denote the truncation function $x\cdot \ind\{a \leq x \leq b\} + a\cdot \ind\{x<a\} + b\cdot \ind\{x>b\}$.

\section{Related Work}


In this section, we will review prior works that are most relevant to ours.

\noindent\textbf{Heterogeneous linear bandits.}
Linear bandits have been studied for a long time. Most of existing works focus on the homogeneous linear bandits where the noise distributions at different round are identical \citep{auer2002using, chu2011contextual, li2010contextual,  dani2008stochastic, abbasi2011improved, li2019nearly, li2019tight}. Recently, a series of works focus on the heterogeneous linear bandits where the noise distribution changes over time. \citet{LaCrSze15} assumed that the noise distribution is Bernoulli and proposed an algorithm with an $\tilde O(d\sqrt{K})$ regret. \citet{kirschner2018information} assumed that the noise at $k$-th round is $\sigma_k^2$-sub-Gaussian, and they proposed a weighted ridge regression-based algorithm with an $\tilde O(d\sqrt{\sum_{k=1}^K\sigma_k^2})$ regret. A recent line of works assume the variance of the noise at $k$-th round is bounded by $\sigma_k^2$. Under this assumption, \citet{zhang2021improved} proposed a VOFUL algorithm with an $\tilde O(d^{4.5}\sqrt{\sum_{k=1}^K\sigma_k^2}+d^5)$ regret. \citet{kim2021improved} proposed a VOFUL2 algorithm with an $\tilde O(d^{1.5}\sqrt{\sum_{k=1}^K\sigma_k^2}+d^2)$ regret. Both VOFUL and VOFUL2 do not need to know the variance information. However, they are computationally inefficient since they need to work with nonconvex confidence sets defined by a series of second-order constraints, and do not propose a polynomial-time algorithm to solve the maximization problem over these sets. With the variance information, \citet{zhou2021nearly} proposed a computationally efficient WeightedOFUL with an $\tilde O(d\sqrt{\sum_{k=1}^K\sigma_k^2} + \sqrt{dK} + d)$ regret. Our work is under the same assumptions as \citet{zhou2021nearly} and improves the regret for linear bandits.

\noindent\textbf{Horizon-free tabular RL.} 
RL is widely believed to be harder than contextual bandits problem due to its long planning horizon and the unknown state transitions. For tabular RL, under the assumption that the total reward obtained by any policy is upper bounded by $1$, \citet{jiang2018open} conjectured that any algorithm to find an $\epsilon$-optimal policy needs to have a polynomial dependence on the planning horizon $H$ in the sample complexity. Such a conjecture was firstly refuted by \citet{wang2020long} by proposing a horizon-free algorithm with an $\tilde O(|\cS|^5|\cA|^4\epsilon^{-2}\text{polylog}(H))$ sample complexity that depends on $H$ polylogarithmically, where $\epsilon$ is the target sub-optimality of the policy, $\cS$ is the state space and $\cA$ is the action space. \citet{zhang2021reinforcement} proposed a near-optimal algorithm with an improved regret $O((\sqrt{|\cS||\cA|K}+|\cS|^2|\cA|)\text{polylog}(H))$ and sample complexity. Similar regret/sample complexity guarantees with a polylogarithmic $H$ dependence have also been established under different RL settings \citep{zhang2020nearly, ren2021nearly, tarbouriech2021stochastic}. Recently \citet{li2022settling, zhang2022horizon} further proposed algorithms with $H$-independent regret/sample complexity guarantees. However, all the above works are limited to tabular RL. Our work proposes an algorithm with a regret bound that depends on $H$ polylogarithmically for linear mixture MDPs, which extends these horizon-free tabular RL algorithms.

\noindent\textbf{RL with linear function approximation.}
Recent years have witnessed a trend on RL with linear function approximation \citep[e.g.,][]{jiang2017contextual,dann2018oracle, yang2019sample, jin2019provably, wang2019optimism,  du2019good,  sun2019model,  zanette2020frequentist,zanette2020learning,weisz2020exponential,yang2019reinforcement, modi2019sample, jia2020model, ayoub2020model, zhou2020provably}. All these works assume that the MDP enjoys some linear representation and propose different statistical and computational complexities which depend on the dimension of the linear representation. Among these assumptions, our work falls into the category of \emph{linear mixture MDP} which assumes that the transition dynamic can be represented as a linear combination of several basis transition probability functions \citep{yang2019reinforcement, modi2019sample, jia2020model, ayoub2020model, zhou2020provably}. Previous algorithms for linear mixture MDPs either suffer from a polynomial dependence on the episode horizon $H$ \citep{yang2019reinforcement, modi2019sample, jia2020model, ayoub2020model,  zhou2020provably, cai2019provably, he2020logarithmic, zhou2021nearly} or do not have a computationally efficient implementation \citep{zhang2021improved, kim2021improved}. Our work achieves the best of both worlds for the first time.

\section{Preliminaries}\label{section 3}

\subsection{Heterogeneous linear bandits}\label{sec:banditintro}

We consider the same heterogeneous linear bandits as studied in \citet{zhou2021nearly}. Let $\{\cD_k\}_{k=1}^\infty$ be decision sets that are fixed. At each round $k$, the agent selects an action $\ab_k \in \cD_k$ satisfying $\|\ab_k\|_2 \leq A$, then receives a reward $r_k$ provided by the environment. Specifically, $r_k$ is generated by $r_k = \la \btheta^*, \ab_k\ra + \epsilon_k$, where $\btheta^* \in \RR^d$ is an unknown vector, and $\epsilon_k$ is a random noise satisfying
\begin{align}
    \forall k,\ |\epsilon_k| \leq R,\ \EE[\epsilon_k|\ab_{1:k}, \epsilon_{1:k-1}] = 0,\ \EE [\epsilon_k^2|\ab_{1:k}, \epsilon_{1:k-1}] \leq \sigma_k^2,\notag
\end{align}
where $\sigma_k$ is an upper bound of the variance of the noise $\epsilon_k$ that are observable to the agent. We assume that $\sigma_k$ is $(\ab_{1:k}, \epsilon_{1:k-1})$-measurable.
The agent aims to minimize the \emph{pseudo-regret} defined as follows:
\begin{align}
    \text{Regret}(K) = \sum_{k=1}^K [\la \ab_k^*, \btheta^*\ra - \la \ab_k, \btheta^*\ra],\ \text{where}\ \ab_k^* = \argmax_{\ab \in \cD_k} \la \ab, \btheta^*\ra.\notag
\end{align}

\subsection{Episodic reinforcement learning}
We also study RL with linear function approximation for episodic linear mixture MDPs. We introduce the necessary definitions of MDPs here. The reader can refer to \cite{puterman2014Markov} for more details.

\noindent\textbf{Episodic MDP.} We denote a homogeneous, episodic MDP by a tuple $M=M(\cS, \cA, H,\reward, \PP)$, where $\cS$ is the state space and $\cA$ is the action space, $H$ is the length of the episode, $\reward: \cS \times \cA \rightarrow [0,1]$ is the deterministic reward function, and $\PP$ is the transition probability function. For the sake of simplicity, we restrict ourselves to countable state space and finite action space. A policy $\pi = \{\pi_h\}_{h=1}^H$ is a collection of $H$ functions, where each of them maps a state $s$ to an action $a$.

\noindent\textbf{Value function and regret.}
For $(s,a)\in \cS \times \cA$, 
we define the action-value function $\qvalue_h^{\pi}(s,a)$ and (state) value function $\vvalue_h^\pi(s)$ as follows:
\begin{align}
&\qvalue_h^{\pi}(s,a) = \EE\bigg[\sum_{h' = h}^H \reward(s_{h'}, a_{h'})\bigg|s_{h} = s, a_h = a, s_{h'} \sim \PP(\cdot|s_{h'-1}, a_{h'-1}), a_{h'} = \pi_{h'}(s_{h'}) \bigg],\notag \\
&\vvalue_h^{\pi}(s) = \qvalue_h^{\pi}(s,\pi_h(s)),\ 
\vvalue_{H+1}^{\pi}(s) = 0.\notag
\end{align}
The optimal value function $V_h^*(\cdot)$ and the optimal action-value function $\qvalue_h^*(\cdot, \cdot)$ are defined by $V^*_h(s) = \sup_{\pi}\vvalue_h^{\pi}(s)$ and $\qvalue_h^*(s,a) = \sup_{\pi}\qvalue_h^{\pi}(s,a)$, respectively. For any function $\vvalue: \cS \rightarrow \RR$, we introduce the following shorthands to denote the conditional variance of $V$ at $\PP(\cdot|s,a)$:
\begin{align*}
[\PP \vvalue](s,a) & =\EE_{s' \sim \PP(\cdot|s,a)}\vvalue(s'),\ 
[\var\vvalue](s,a)  = [\PP \vvalue^2](s,a) - ([\PP \vvalue](s,a))^2,
\end{align*}
 where $\vvalue^2$ stands for the function whose value at $s$ is $\vvalue^2(s)$. Using this notation, the Bellman equations for policy $\pi$ and the Bellman optimality equation can be written as
\begin{align}
    \qvalue_h^{\pi}(s,a) = \reward(s,a) +[\PP\vvalue_{h+1}^\pi](s,a),\ \qvalue_h^*(s,a) = \reward(s,a) +[\PP\vvalue_{h+1}^*](s,a).\notag
\end{align}
The goal is to minimize the $K$-episode regret defined as follows:
\begin{align}
    \text{Regret}(K) = \sum_{k=1}^K \big[\vvalue_1^*(s_1^k) - \vvalue_{1}^{\pi^k}(s_1^k)\big].\notag
\end{align}
In this paper, we focus on proving high probability bounds on the regret $\text{Regret}(K)$.

In this work we make the following assumptions. The first assumption assumes that for any policy, the accumulated reward of an episode is upper bounded by 1, which has been considered in previous works \citep{krishnamurthy2016pac,jiang2018open}. The accumulated reward assumption ensures that the only factor that can affect the final statistical complexity is the planning difficulty brought by the episode length, rather than the scale of the reward. 
\begin{assumption}[Bounded total reward]\label{ass:totalreward}
For any policy $\pi$, let $\{s_h, a_h\}_{h=1}^H$ be any states and actions satisfying $a_h = \pi_h(s_h)$ and $s_{h+1} \sim \PP(\cdot|s_h, a_h)$. Then we have $0 \leq \sum_{h=1}^H r(s_h, a_h) \leq 1$.
\end{assumption}
Next assumption assumes that the transition dynamic enjoys a linearized representation w.r.t. some feature mapping. We define the \emph{linear mixture MDPs} \citep{jia2020model, ayoub2020model,zhou2020provably} as follows.
\begin{assumption}[Linear mixture MDP]\label{assumption-linear}
$M$ is an episodic $\pnorm$-bounded linear mixture MDP, such that there exists a vector $\btheta^* \in \RR^d$ and $\bphi(\cdot|\cdot, \cdot)$ such that $\PP(s'|s,a) = \la \bphi(s'|s,a), \btheta^*\ra$ for any state-action-next-state triplet $(s,a,s') \in \cS \times \cA \times \cS$. Meanwhile, $\|\btheta^*\|_2 \leq \pnorm$ and for any bounded function $\vvalue: \cS \rightarrow [0,1]$ and any tuple $(s,a)\in \cS \times \cA$, we have 
\begin{align}
    \|\bphi_{{\vvalue}}(s,a)\|_2 \leq 1,\text{where}\ 
    \bphi_{{\vvalue}}(s,a) =
    \sum_{s'\in \cS}\bphi(s'|s,a)\vvalue(s').\notag
\end{align}
Lastly, for any $\vvalue: \cS \rightarrow [0,1]$, $\bphi_V$ can be calculated efficiently within $\cO$ time.
\end{assumption}
\begin{remark}\label{remark:linear}
A key property of linear mixture MDP is that for any function $\vvalue:\cS \rightarrow \RR$ and any state-action pair $(s,a)$, the conditional expectation of $V$ over $\PP(\cdot|s,a)$ is a linear function of $\btheta^*$, i.e., $[\PP \vvalue](s,a) = \la \bphi_V(s,a), \btheta^*\ra$. Meanwhile, the conditional variance of $V$ over $\PP(\cdot|s,a)$ is a quadratic function of $\btheta^*$, i.e., $[\var\vvalue](s,a) = \la \bphi_{V^2}(s,a), \btheta^*\ra - [\la \bphi_{V}(s,a), \btheta^*\ra]^2$.  
\end{remark}

\begin{remark}\label{remark:efficient}
For a general class of $\bphi$, $\bphi_V(s,a): \cS \times \cA \rightarrow \RR^d$ can be computed efficiently for any $(s,a) \in \cS \times \cA$ if $V:\cS \rightarrow \RR$ can be computed efficiently. For instance, $\bphi(s'|s,a) = \eb_{s',s,a} \in \RR^{|\cS|^2|\cA|}$ and $\bphi(s'|s,a) = \bpsi(s')\odot \bxi(s,a)$, where $\bpsi, \bxi$ are two sub feature functions. More discussions are referred to \citet{zhou2021nearly}.
\end{remark}

\section{Computationally Efficient Variance-Aware Linear Bandits}
\label{sec:linearbandit}

In this section, we propose our algorithm $\algbandit$ in Algorithm \ref{algorithm:reweightbandit} for the heterogeneous linear bandits introduced in Section \ref{sec:banditintro}. $\algbandit$ adopts the \emph{weighted ridge regression} estimator used by WeightedOFUL \citep{zhou2021nearly}, but uses a refined weight. It first computes a weighted estimate of $\btheta^*$, denoted by $\hat\btheta_k$, based on previous contexts and rewards, where the weights $\bar\sigma_k$ are computed by the noise variance $\sigma_k$. Then $\algbandit$ constructs the confidence set of $\btheta^*$, denoted by $\hat\cC_k$, estimates the reward $\la\ab, \btheta\ra$ for $\btheta \in \hat\cC_k$, and selects the arm that maximizes the estimated reward optimistically. The selection rule of $\ab_k$ is identical to selecting the best arm w.r.t. to their upper confidence bound, i.e., $\ab_k \leftarrow\argmax_{\ab \in \cD_k}\la \ab, \hat\btheta_k\ra + \hat\beta_k\|\ab\|_{\hat\bSigma_k^{-1}}$ \citep{li2010contextual}. $\algbandit$ constructs $\bar\sigma_k^2$ as the maximization of the variance, a constant, and the uncertainty $\|\ab_k\|_{\hat\bSigma_k^{-1}}$, as defined in \eqref{def:banditvar}. 
Note that \citet{zhou2021nearly} proposed a \emph{variance-aware} weight $\bar\sigma_k = \max\{\sigma_k, \alpha\}$ for heterogeneous linear bandits. The additional uncertainty term in our weight enables us to build a tighter confidence set $\hat\cC_k$ since the uncertainty of arms can be leveraged by a tighter Bernstein-type concentration inequality (See Section \ref{sec:sketch} for more details). Meanwhile, we notice that \citet{he2022corrupt} proposed a pure \emph{uncertainty-aware} weight $\bar\sigma_k = \max\{\alpha,\gamma\|\ab_k\|_{\hat\bSigma_k^{-1}}^{1/2}\}$ to deal with the corruption in contextual linear bandits, which serves a different purpose compared to our setting (there is no corruption in our bandit model).

\begin{algorithm}[t]
	\caption{$\algbandit$}\label{algorithm:reweightbandit}
	\begin{algorithmic}[1]
	\REQUIRE Regularization parameter $\lambda>0$, 
	and $\pnorm$, 
	an upper bound on the $\ell_2$-norm of $\btheta^*$, confidence radius $\hat\beta_k$
	\STATE $\hat\bSigma_1 \leftarrow \lambda \Ib$, $\hat\bbb_1 \leftarrow \zero$, $\hat\btheta_1 \leftarrow \zero$, $\hat\beta_1 = \sqrt{\lambda}\pnorm$
	\FOR{$k=1,\ldots, K$}
	\STATE Let $\hat\cC_k \leftarrow \{\btheta:\|\hat\bSigma_{k}^{1/2}(\btheta - \hat\btheta_k)\|_2 \leq \hat\beta_k\}$, observe $\cD_k$
	\STATE Set $(\ab_k, \btheta_k) \leftarrow \argmax_{\ab \in \cD_k, \btheta \in \hat\cC_k} \la \ab, \btheta\ra$ 
	\STATE Observe $(r_k,\sigma_k)$, set $\bar\sigma_k$ as 
	\begin{align}
	    &\bar\sigma_k \leftarrow \max\{\sigma_k, \alpha, \gamma\|\ab_k\|_{\hat\bSigma_k^{-1}}^{1/2}\}\label{def:banditvar}
	\end{align}
	\STATE $\hat\bSigma_{k+1} \leftarrow \hat\bSigma_k + \ab_k\ab_k^\top/\bar\sigma_k^2$, $\hat\bbb_{k+1} \leftarrow \hat\bbb_k + r_k\ab_k/\bar\sigma_k^2$, $\hat\btheta_{k+1}\leftarrow \hat\bSigma_{k+1}^{-1}\hat\bbb_{k+1}$\label{alg:reweightbandit}
	\ENDFOR
	\end{algorithmic}
\end{algorithm}
The following theorem gives the regret bound of $\algbandit$.
\begin{theorem}\label{coro:linearregret}
Let $0<\delta<1$. Suppose that for all $k \geq 1$ and all $\ab \in \cD_k$, $\la \ab, \btheta^*\ra \in [-1, 1]$, $\|\btheta^*\|_2 \leq \pnorm$
and $\{\hat\beta_k\}_{k \geq 1}$ are set to
\begin{align}
    \hat\beta_k& = 12\sqrt{d\log(1+kA^2/(\alpha^2d\lambda))\log(32(\log(\gamma^2/\alpha)+1)k^2/\delta)} \notag \\
    &\quad + 30\log(32(\log(\gamma^2/\alpha)+1)k^2/\delta)R/\gamma^2 + \sqrt{\lambda}\pnorm.
\label{eq:defbanditbeta}
\end{align}
 Then with probability at least $1-\delta$, the regret of $\algbandit$ is bounded by
\begin{align}
    \text{Regret}(K)
    & \leq 4d \iota + 4d\gamma^2\hat\beta_K\iota + 4\hat\beta_K\sqrt{\textstyle{\sum}_{k=1}^K\sigma_k^2 + K\alpha^2}\sqrt{d \iota},\label{eq:cororegret}
\end{align}
where $\iota = \log(1+KA^2/(d\lambda\alpha^2))$.
Moreover, setting $\alpha = 1/\sqrt{K}, \gamma = R^{1/2}/d^{1/4}$ and $\lambda = d/B^2$ yields a high probability regret $\text{Regret}(K) = \tilde O(d\sqrt{\sum_{k=1}^K\sigma_k^2} +dR +d)$.
\end{theorem}
\begin{remark}\label{rmk:limit}
Treating $R$ as a constant, the regret of $\algbandit$ becomes $\tilde O(d\sqrt{\sum_{k=1}^K\sigma_k^2} +d)$. It strictly outperforms the $\tilde O(d\sqrt{\sum_{k=1}^K\sigma_k^2} + \sqrt{dK} + d)$ regret achieved in \citet{zhou2021nearly}. Compared with the $\tilde O(d^{4.5}\sqrt{\sum_{k=1}^K\sigma_k^2} + d^5)$ regret by VOFUL \citep{zhang2021improved} and the $\tilde O(d^{1.5}\sqrt{\sum_{k=1}^K\sigma_k^2} + d^2)$ regret by VOFUL2 \citep{kim2021improved}, the regret of $\algbandit$ has a better dependence on $d$, and $\algbandit$ is computationally efficient. It is worth noting that both VOFUL and VOFUL2 do not need to know the variance $\sigma_k^2$ while our algorithm does. Whether there exists an algorithm that can achieve the $\tilde O(d\sqrt{\sum_{k=1}^K\sigma_k^2} + \sqrt{dK} + d)$ regret without knowing the variance information remains an open problem.
\end{remark}

\subsection{Proof sketch}\label{sec:sketch}
We give a proof sketch of Theorem \ref{coro:linearregret} along with two key lemmas that are pivotal to obtain the improved regret. To begin with, we first show that $\btheta^*$ belongs to confidence balls centering at $\hat\btheta_k$ with radius $\hat\beta_k$. This can be proved by the following theorem, which is an improved version of the Bernstein-type self-normalized martingale inequality proposed by \citet{zhou2021nearly}. 
\begin{theorem}\label{lemma:concentration_variance}
Let $\{\cG_k\}_{k=1}^\infty$ be a filtration, and $\{\xb_k,\eta_k\}_{k\ge 1}$ be a stochastic process such that
$\xb_k \in \RR^d$ is $\cG_k$-measurable and $\eta_k \in \RR$ is $\cG_{k+1}$-measurable.
Let $L,\sigma,\lambda, \epsilon>0$, $\bmu^*\in \RR^d$. 
For $k\ge 1$, 
let $y_k = \la \bmu^*, \xb_k\ra + \eta_k$ and
suppose that $\eta_k, \xb_k$ also satisfy 
\begin{align}
    \EE[\eta_k|\cG_k] = 0,\ \EE [\eta_k^2|\cG_k] \leq \sigma^2,\  |\eta_k| \leq R,\,\|\xb_k\|_2 \leq L.
\end{align}
For $k\ge 1$, let $\Zb_k = \lambda\Ib + \sum_{i=1}^{k} \xb_i\xb_i^\top$, $\bbb_k = \sum_{i=1}^{k}y_i\xb_i$, $\bmu_k = \Zb_k^{-1}\bbb_k$, and
\begin{small}
\begin{align}
    \beta_k &= 12\sqrt{\sigma^2d\log(1+kL^2/(d\lambda))\log(32(\log(R/\epsilon)+1)k^2/\delta)} \notag \\
    &\quad + 24\log(32(\log(R/\epsilon)+1)k^2/\delta)\max_{1 \leq i \leq k} \{|\eta_i|\min\{1, \|\xb_i\|_{\Zb_{i-1}^{-1}}\}\} + 6\log(32(\log(R/\epsilon)+1)k^2/\delta)\epsilon.\notag
\end{align}
\end{small}
Then, for any $0 <\delta<1$, we have with probability at least $1-\delta$ that, 
\begin{align}
    \forall k\geq 1,\ \big\|\textstyle{\sum}_{i=1}^{k} \xb_i \eta_i\big\|_{\Zb_k^{-1}} \leq \beta_k,\ \|\bmu_k - \bmu^*\|_{\Zb_k} \leq \beta_k + \sqrt{\lambda}\|\bmu^*\|_2,\notag
\end{align}
\end{theorem}
Note that $\hat\btheta_k$ can be regarded as $\bmu_{k-1}$ in Theorem \ref{lemma:concentration_variance} with $\epsilon = R/\gamma^2$, $\xb_k = \ab_k/\bar\sigma_k$, $y_k = r_k/\bar\sigma_k$, $\eta_k = \epsilon_k/\bar\sigma_k$, $\Zb_k = \hat\bSigma_{k+1}$ and $\bmu^* = \btheta^*$. With the help of the weight $\bar\sigma_k$, the variance of $\eta_k$ is upper bounded by 1 (since $\bar\sigma_k \geq \sigma_k$) and  $|\eta_k|\min\{1, \|\xb_k\|_{\Zb_k^{-1}}\} \leq |\epsilon_k|\|\ab_k\|_{\hat\bSigma_k^{-1}}/\bar\sigma_k^2 \leq R/\gamma^2$ (since $\bar\sigma_k^2 \geq \gamma^2\|\ab_k\|_{\hat\bSigma_k^{-1}}$). Therefore, by Theorem \ref{lemma:concentration_variance}, w.h.p. $\btheta^* \in \hat\cC_k$. Following the standard procedure to bound the regret of the optimistic algorithm \citep{abbasi2011improved}, we have
\begin{small}
\begin{align}
    \text{Regret}(K) = \sum_{k=1}^K \la \ab_k^* - \ab_k, \btheta^*\ra \leq  2\sum_{k=1}^K \min\{1, \hat\beta_k \|\ab_k\|_{\hat\bSigma_k^{-1}}\}.\label{rrr:1}
\end{align}
\end{small}
Next lemma gives an upper bound of \eqref{rrr:1}. 
\begin{lemma}\label{lemma:keysum:temp}
Let $\{\sigma_k, \hat\beta_k\}_{k \geq 1}$ be a sequence of non-negative numbers, $\alpha, \gamma>0$, $\{\ab_k\}_{k \geq 1} \subset \RR^d$ and $\|\ab_k\|_2 \leq A$. Let $\{\bar\sigma_k\}_{k \geq 1}$ and $\{\hat\bSigma_k\}_{k \geq 1}$ be (recursively) defined as follows: $\hat\bSigma_1 = \lambda\Ib$,\ 
\begin{align}
    \forall k \geq 1,\ \bar\sigma_k = \max\{\sigma_k, \alpha, \gamma\|\ab_k\|_{\hat\bSigma_k^{-1}}^{1/2}\},\ \hat\bSigma_{k+1} = \hat\bSigma_k + \ab_k\ab_k^\top/\bar\sigma_k^2.\notag
\end{align}
Let $\iota = \log(1+KA^2/(d\lambda\alpha^2))$. Then we have
\begin{align}
    \sum_{k=1}^K\min\Big\{1, \hat\beta_k\|\ab_k\|_{\hat\bSigma_k^{-1}}\Big\} &\leq 2d \iota +2\max_{k \in [K]}\hat\beta_k \gamma^2d\iota+ 2\sqrt{d \iota}\sqrt{\sum_{k=1}^K\hat\beta_k^2(\sigma_k^2 + \alpha^2)}\notag.
\end{align}
\end{lemma}
By Lemma \ref{lemma:keysum:temp}, the regret of $\algbandit$ can be bounded by
\begin{small}
\begin{align}
    \text{Regret}(K) = \tilde O\big(d+\sqrt{d}\hat\beta_K\big(\sqrt{\textstyle{\sum}_{k=1}^K \sigma_k^2 + K\alpha^2} + \sqrt{d}\gamma^2\big)\big)\label{compare1}
\end{align}
\end{small}
with $\hat\beta_K= \tilde O(\sqrt{d} + R/\gamma^2 + \sqrt{\lambda}\pnorm)$, which finishes the proof.

Here we compare the regret of $\algbandit$ and the regret of WeightedOFUL \citep{zhou2021nearly} (which chooses the weight $\bar\sigma_k = \max\{\sigma_k, \alpha\}$) to see where the improvement comes from. In particular, \citet{zhou2021nearly} proved the regret of WeightedOFUL as follows
\begin{small}
\begin{align}
    \text{Regret}(K) = \tilde O\big(d+\sqrt{d}\hat\beta_K^\alpha\sqrt{\textstyle{\sum}_{k=1}^K\sigma_k^2 + K\alpha^2}\big)\label{compare2}
\end{align}
\end{small}
with $\hat\beta_K^\alpha = \tilde O(\sqrt{d} + R/\alpha + \sqrt{\lambda}\pnorm)$. At the first glance, both \eqref{compare1} and \eqref{compare2} have a $\sqrt{K\alpha^2}$ term. However, thanks to the design of $\bar\sigma_k$ and Theorem \ref{lemma:concentration_variance}, the $\sqrt{K\alpha^2}$ term in \eqref{compare1} can be shaved by choosing a small enough $\alpha$, while this term in \eqref{compare2} cannot be shaved due to the existence of a $R/\alpha$ term in $\hat\beta_K^\alpha$.

\section{Computationally Efficient Horizon-Free RL for Linear Mixture MDPs}\label{section:finite_main}

\begin{algorithm}[t]
	\caption{$\algmdp$}\label{algorithm:finite}
	\begin{algorithmic}[1]
	\REQUIRE Regularization parameter $\lambda$, an upper bound $\pnorm$ of the $\ell_2$-norm of $\btheta^*$, confidence radius $\{\hat\beta_k\}_{k \geq 1}$, level $M$, variance parameters $\alpha, \gamma$, $\seq{M} = \{0,\dots, M-1\}$
	\STATE For $m \in \seq{M}$, set $\hat\btheta_{1,m} \leftarrow\zero$, $\tilde\bSigma_{0,H+1,m}\leftarrow \lambda\Ib$, $\tilde\bbb_{0,H+1,m} \leftarrow \zero$. Set $\vvalue_{1, H+1}(\cdot) \leftarrow 0$
	\FOR{$k=1,\ldots, K$}
		\FOR{$h = H,\dots, 1$}
	\STATE Set $\qvalue_{k,h}(\cdot, \cdot)\leftarrow \Big[ \reward(\cdot, \cdot) + \big\la \hat\btheta_{k,0}, \bphi_{\vvalue_{k, h+1}}(\cdot, \cdot) \big\ra + \hat\beta_k \Big\|\hat \bSigma_{k,0}^{-1/2} \bphi_{\vvalue_{k, h+1}}(\cdot, \cdot)\Big\|_2\Big]_{[0, 1]}$
\STATE Set $\pi_h^{k}(\cdot) \leftarrow \argmax_{a \in \cA}\qvalue_{k,h}(\cdot, a)$
\STATE Set $\vvalue_{k,h}(\cdot) \leftarrow \max_{a \in \cA}\qvalue_{k,h}(\cdot, a)$
	\ENDFOR
		\STATE	Receive $s_1^k$. For $m \in \seq{M}$, set $\tilde\bSigma_{k,1,m} \leftarrow \tilde\bSigma_{k-1,H+1,m}$
	\FOR{$h = 1,\dots, H$}
	\STATE Take action $a_h^k  \leftarrow \pi_h^k(s_h^k)$, receive $s_{h+1}^k \sim \PP(\cdot|s_h^k, a_h^k)$. 
	\STATE For $m \in \seq{M}$, denote $\bphi_{k,h,m} = \bphi_{\vvalue_{k,h+1}^{2^{m}}}(s_h^k, a_h^k)$.
	\STATE  Set $\{\bar\sigma_{k,h,m}\}_{m \in \seq{M}} \leftarrow$Algorithm \ref{algorithm:variance}($\{\bphi_{k,h,m},\hat\btheta_{k,m},\tilde\bSigma_{k,h,m}, \hat\bSigma_{k,m}\}_{m \in \seq{M}}$, $\hat\beta_k$, $\alpha, \gamma$)
	\STATE For $m \in \seq{M}$, set $\tilde\bSigma_{k,h+1,m} \leftarrow \tilde\bSigma_{k,h,m} + \bphi_{k,h,m}\bphi_{k,h,m}^\top/\bar\sigma_{k,h,m}^2$
	\STATE For $m \in \seq{M}$, set $\tilde\bbb_{k,h+1,m}\leftarrow \tilde\bbb_{k,h,m} + \bphi_{k,h,m}\vvalue_{k,h+1}^{2^m}(s_{h+1}^k)/\bar\sigma_{k,h,m}^2$
	\ENDFOR
	\STATE  For $m \in \seq{M}$, set $\hat\bSigma_{k+1,m}\leftarrow \tilde\bSigma_{k,H+1,m},\hat\bbb_{k+1,m} \leftarrow \tilde\bbb_{k,H+1,m}, \hat\btheta_{k+1,m} \leftarrow \hat\bSigma_{k+1,m}^{-1}\hat\bbb_{k+1,m}$
	\ENDFOR
	\end{algorithmic}
\end{algorithm}
In this section, we propose a horizon-free RL algorithm $\algmdp$ in Algorithm \ref{algorithm:finite}. 
$\algmdp$ follows the \emph{value targeted regression (VTR)} framework proposed by \citet{jia2020model, ayoub2020model} to learn the linear mixture MDP. In detail, following the observation in Remark \ref{remark:linear}, VTR estimates $\btheta^*$ by solving a regression problem over predictors/contexts $\bphi_{k,h,0} = \bphi_{V_{k,h+1}}(s_h^k, a_h^k)$ and responses $V_{k,h+1}(s_{h+1}^k)$. Specifically, $\algmdp$ takes the estimate $\hat\btheta_{k,0}$ as the solution to the following weighted regression problem:
\begin{align}
    \hat\btheta_{k,0} = \argmin_{\btheta \in \RR^d}\lambda\|\btheta\|_2^2 + \textstyle{\sum_{j = 1}^{k-1} \sum_{h=1}^H}\big[\big\la\bphi_{j,h,0}, \btheta\big\ra - \vvalue_{j, h+1}(s_{h+1}^j)\big]^2/\bar\sigma_{j,h,0}^2,
\end{align}
where $\bar\sigma_{j,h,0}$ is the upper bound of the conditional variance $[\var V_{k,h+1}](s_h^k, a_h^k)$. Such a weighted regression scheme has been adapted by UCRL-VTR$^+$ in \citet{zhou2021nearly}. With $\hat\btheta_{k,0}$, $\algmdp$ then constructs the optimistic estimates $Q_{k,h}$ (resp. $V_{k,h}$) of the optimal value functions $Q^*_{h}$ (resp. $V^*_{h}$) and takes actions optimistically. Note that $\hat\btheta_{k,0}$ is updated at the end of each episode. We highlight several improved algorithm designs of $\algmdp$ compared to UCRL-VTR$^+$  as follows.

\noindent\textbf{Improved weighted linear regression estimator.} $\algmdp$ sets $\bar\sigma_{j,h,0}$ similar to the weight used in $\algbandit$. Assuming that the conditional variance $[\var V_{k,h+1}](s_h^k, a_h^k)$ can be computed for any value function $\vvalue$ and state action pair $(s,a)$, then the weight can be set as 
\begin{align}
    \bar\sigma_{k,h,0}^2 = \max\{[\var V_{k,h+1}](s_h^k, a_h^k), \alpha^2, \gamma^2\|\tilde\bSigma_{k,h,0}^{-1/2}\bphi_{k,h,0}\|_2\},\notag
\end{align}
where $\tilde\bSigma_{k,h,0}$ is the weighted sample covariance matrix of $\bphi_{k,h,0}$ up to $k$-th episode and $h$-th stage. However, the true variance is not accessible since $\PP(\cdot|s,a)$ is unknown. Therefore, $\algmdp$ replaces $[\var V_{k,h+1}](s_h^k, a_h^k)$ with its estimate $[\bar \var_{k,0} V_{k,h+1}](s_h^k, a_h^k)$ and an error bound $E_{k,h,0}$ satisfying $[\bar \var_{k,0} V_{k,h+1}](s_h^k, a_h^k) + E_{k,h,0} \geq [\var V_{k,h+1}](s_h^k, a_h^k)$ with high probability. Thanks to the fact that $[\var V_{k,h+1}](s_h^k, a_h^k)$ is a quadratic function of $\btheta^*$ as illustrated in Remark \ref{remark:linear}, $[\bar \var_{k,0} V_{k,h+1}](s_h^k, a_h^k)$ can be estimated as follows:
\begin{align}
    [\bar \var_{k,0}\vvalue_{k, h+1}](s_h^k, a_h^k)
    & = \big[\big\la\bphi_{k,h,1}, \hat\btheta_{k,1}\big\ra\big]_{[0,1]}  - \big[\big\la \bphi_{k,h,0}, \hat\btheta_{k,0}\big\ra\big]_{[0,1]}^2,\label{eq:help1}
\end{align}
where $\hat\btheta_{k,1}$ is the solution to some regression problem over predictors/contexts $\bphi_{k,h,1} = \bphi_{V_{k,h+1}^2}(s_h^k, a_h^k)$ and responses $V_{k,h+1}^2(s_{h+1}^k)$. 

\noindent\textbf{Higher-order moment regression.}
To obtain a better estimate, it is natural to set $\hat\btheta_{k,1}$ as the solution to the weighted regression problem on $\bphi_{k,h,1}$ and $V_{k,h+1}^2(s_{h+1}^k)$ with weight $\bar\sigma_{k,h,1}$. Here $\bar\sigma_{k,h,1}$ is constructed in a similar way to $\bar\sigma_{k,h,0}$, which relies on the conditional variance of $[\var\vvalue_{k,h+1}^2](s_h^k, a_h^k)$. By repeating this process, we recursively estimate the conditional $2^m$-th moment of $\vvalue_{k,h+1}$ by its variance, which is the conditional $2^{m+1}$-th moment of $\vvalue_{k,h+1}$. 
It is worth noting that the idea of high-order recursive estimation has been used in \citet{li2020breaking} and later in \citet{zhang2021reinforcement, zhang2021improved} to achieve horizon-free regret/sample complexity guarantees. Similar recursive analysis also appeared in \citet{lattimore2012pac}. The estimated conditional moment $[\bar\var_{k,m}\vvalue_{k,h+1}^{2^m}](s_h^k, a_h^k)$ relies on $\hat\btheta_{k,m+1}$ and $\hat\btheta_{k,m}$, and $\la\bphi_{k,h,m+1}, \hat\btheta_{k,m+1}\ra$ serves as the estimate of the higher-moment $[\var\vvalue_{k,h+1}^{2^{m+1}}](s_h^k, a_h^k)$. The detailed constructions for the high-order moment estimator are summarized in Algorithm \ref{algorithm:variance}.

\noindent\textbf{Computational complexity of $\algmdp$}
At each episode $k$ and each stage $h$, HF-UCRL-VTR$^+$ needs to compute $\{\bphi_{k,h,m}\}_{m \in \seq{M}}$ and $\{\bar\sigma_{k,h,m}\}_{m \in \seq{M}}$, and update $\{\tilde\bSigma_{k,h+1,m}\}_{m \in \seq{M}}$. According to Algorithm \ref{algorithm:variance}, $\{\bar\sigma_{k,h,m}\}_{m \in \seq{M}}$ can be computed in $O(Md^2)$ time as they only require the computation of the inner-product between vectors and the inner-product between an inversion of matrix and an vector. For $\{\bphi_{k,h,m}\}_{m \in \seq{M}}$, they can be computed within $O(\cO\cdot M)$ time. Finally, to selection actions based on $\pi_h^k$, $\algmdp$ needs to compute $|\cA|$ number of action-value function $Q_{k,h}$, while each of them needs to compute the $\bphi_V(\cdot, \cdot)$ within $\cO$ time and the inner-product between an inversion of a matrix and a vector by $O(d^2)$ time. Therefore, the total amount of time $\algmdp$ takes is $O(KHMd^2 + KH\cO M + |\cA|KH\cO + |\cA|KHd^2)$.

\begin{remark}
Compared with $\algmdp$, VOFUL/VOFUL2 \citep{zhang2021improved, kim2021improved} need to compute the upper bounds of moments as the maximum of the quadratic function $\big[\big\la\bphi_{k,h,m+1}, \btheta\big\ra\big]_{[0,1]}  - \big[\big\la \bphi_{k,h,m}, \btheta\big\ra\big]_{[0,1]}^2$ over a series of implicit confidence sets, which are not implementable.
\end{remark}

\begin{algorithm}[t]
	\caption{High-order moment estimator ($\algvar$)}\label{algorithm:variance}
	\begin{algorithmic}[1]
	\REQUIRE Features $\{\bphi_{k,h,m}\}_{m \in \seq{M}}$, vector estimators $\{\hat\btheta_{k,m}\}_{m \in \seq{M}}$, covariance matrix $\{\tilde\bSigma_{k,h,m}, \hat\bSigma_{k,m}\}_{m \in \seq{M}}$, confidence radius $\hat\beta_k$, $\alpha, \gamma$
	\FOR{$m = 0,\dots, M-2$}
	\STATE Set $[\bar\var_{k,m}\vvalue_{k,h+1}^{2^m}](s_h^k, a_h^k) \leftarrow  \big[\big\la\bphi_{k,h,m+1}, \hat\btheta_{k,m+1}\big\ra\big]_{[0, 1]} -  \big[\big\la \bphi_{k,h,m}, \hat\btheta_{k,m}\big\ra\big]_{[0,1]}^2$
	\STATE Set $\error_{k,h,m} \leftarrow   \min\big\{1,2\hat\beta_k\big\|\hat\bSigma_{k,m}^{-1/2}\bphi_{k,h,m}\big\|_2\big\} + \min\big\{1, \hat\beta_k\big\|\hat\bSigma_{k,m+1}^{-1/2}\bphi_{k,h,m+1}\big\|_2\big\}$
	\STATE Set $\bar\sigma_{k,h,m}^2\leftarrow \max\big\{ [\bar\var_{k,m}\vvalue_{k, h+1}^{2^m}](s_h^k, a_h^k) + \error_{k,h,m}, \alpha^2, \gamma^2\big\|\tilde\bSigma_{k,h,m}^{-1/2}\bphi_{k,h,m}\big\|_2\big\}$
	\ENDFOR
	\STATE  Set $\bar\sigma_{k,h,M-1}^2\leftarrow\max\big\{ 1, \alpha^2, \gamma^2\Big\|\tilde\bSigma_{k,h,M-1}^{-1/2}\bphi_{k,h,M-1}\Big\|_2\big\}$
	\ENSURE $\{\bar\sigma_{k,h,m}\}_{m \in \seq{M}}$
	\end{algorithmic}
\end{algorithm}

We provide the regret bound for $\algmdp$ here.  
\begin{theorem}\label{thm:regret:finite}
Set $M = \log(3KH)/\log 2$. For any $\delta>0$, set $\{\hat\beta_k\}_{k \geq 1}$ as
\begin{small}
\begin{align}
\hat\beta_k&= 12\sqrt{d\log(1+kH/(\alpha^2d\lambda))\log(32(\log(\gamma^2/\alpha)+1)k^2H^2/\delta)} \notag \\
    &\quad + 30\log(32(\log(\gamma^2/\alpha)+1)k^2H^2/\delta)/\gamma^2 + \sqrt{\lambda}\pnorm,\label{def:hatbeta}
\end{align}
\end{small}
then with probability at least $1-(2M+1)\delta$, the regret of Algorithm \ref{algorithm:finite} is bounded by 
\begin{align}
    \text{Regret}(K) &\leq  1728\max\{2\hat\beta_K^2d \iota, \zeta\}+ 48(2d \iota + 2\hat\beta_K \gamma^2d\iota + \hat\beta_K\sqrt{d \iota}\sqrt{Md\iota/2+KH\alpha^2})\notag \\
    &\quad + Md\iota/2 +\big[\sqrt{2\log(1/\delta)}+ 32\max\{2\hat\beta_K\sqrt{d \iota}, \sqrt{2\zeta}\}\big]\sqrt{K},\notag
\end{align}
where $\iota = \log(1+KH/(d\lambda\alpha^2)),\ \zeta = 4\log(4\log(KH)/\delta)$.
Moreover, setting $\alpha = \sqrt{d/(KH)}, \gamma = 1/d^{1/4}$ and $\lambda = d/B^2$ yields a high-probability regret $\text{Regret}(K) = \tilde O(d\sqrt{K} + d^2)$.
\end{theorem}
\begin{remark}
The regret of $\algmdp$ is strictly better than that of VOFUL $\tilde O(d^{4.5}\sqrt{K}+d^9)$  \citet{zhang2021improved}, and it matches the regret of VOFUL2 \citep{kim2021improved}. More importantly, $\algmdp$ is computationally efficient, while there is no efficient implementation of VOFUL/VOFUL2.
\end{remark}
Next theorem provides the regret lower bound and suggests that the regret obtained by $\algmdp$ is near-optimal. The lower bound is proved by constructing hard-instances of linear mixture MDPs following \citet{zhou2020provably, zhou2021nearly, zhang2021reward}.
\begin{theorem}\label{prop:lowerbound}
Let $\pnorm>1$. Then for any algorithm, when $K \geq \max\{3d^2, (d-1)/(192(\pnorm - 1))\}$, there exists a $\pnorm$-bounded linear mixture MDP satisfying Assumptions \ref{ass:totalreward} and \ref{assumption-linear} such that its expected regret $\EE[\text{Regret}(K)]$ is lower bounded by $d\sqrt{K}/(16\sqrt{3})$. 
\end{theorem}
\begin{remark}
When specialized to tabular MDPs where $d = |\cS|^2|\cA|$, $\algmdp$ yields a horizon-free regret $\tilde O(|\cS|^2|\cA|\sqrt{K} + |\cS|^4|\cA|^2)$. Although the regret does not match the near-optimal result $\tilde O(\sqrt{|\cS||\cA|K} + |\cS|^2|\cA|)$ \citep{zhang2021reinforcement}, it is not surprising since $\algmdp$ is designed for a more general MDP class. We leave the design of algorithms that achieve near-optimal regret for both linear mixture MDPs and tabular MDPs simultaneously as a future work. 
\end{remark}

\section{Conclusion}
In this work, we propose a new weighted linear regression estimator that adapts \emph{variance-uncertainty-aware} weights, which can be applied to both heterogeneous linear bandits and linear mixture MDPs. For heterogeneous linear bandits, our $\algbandit$ algorithm achieves an $\tilde O(d\sqrt{\sum_{k=1}^K \sigma_k^2} +d)$ regret in the first $K$ rounds. For linear mixture MDPs, our $\algmdp$ algorithm achieves the near-optimal $\tilde O(d\sqrt{K} +d^2)$ regret. Both of our algorithms are computationally efficient and yield the state-of-the-arts regret results.

\appendix

\section{Proof of the Bernstein-type Concentration Inequality (Theorem~\ref{lemma:concentration_variance})}
We prove Theorem \ref{lemma:concentration_variance} here. First, we provide some new concentration inequalities here, which will be used in the following proof.

\begin{lemma}\label{lemma:newbern}
Let $M,v>0$ be constants. Let $\{x_i\}_{i=1}^n$ be a stochastic process, $\cG_i = \sigma(x_1,\dots, x_i)$ be the $\sigma$-algebra of $x_1,\dots, x_i$. Suppose $\EE[x_i|\cG_{i-1}]=0$, $|x_i|\leq M$ and $\sum_{i=1}^n\EE[x_i^2|\cG_{i-1}]\leq v^2$ almost surely. 
Then for any $\delta, \epsilon>0$, let $\iota = \sqrt{\log((\log(M/\epsilon)+2)/\delta)}$, we have
\begin{align}
    &\PP\bigg(\sum_{i=1}^n x_i \leq 3\iota v + 9\iota^2\max_{i \in [n]}|x_i| + 4\iota^2\epsilon\bigg) > 1- \delta.\notag
\end{align}
\end{lemma}
\begin{proof}
For simplicity, let 
\begin{align}
    S_n = \sum_{i=1}^n x_i,\ \bar S_n = \sqrt{\sum_{i=1}^n\EE[x_i^2|\cG_{i-1}]},\ \bar x_n = \max_{i \in [n]}|x_i|.\notag
\end{align}
We have $\bar S_n < v^2,\bar x_n \leq M$. Set threshold $\epsilon>0$. Set the following sets
\begin{align}
    \Gamma_0 = [0,\epsilon),\ \Gamma_i = [\epsilon\cdot e^{i-1},\epsilon\cdot e^{i}),\ i = 1,\dots, I = \lfloor\log(M/\epsilon)\rfloor+1. \notag
\end{align}
Then there exist $0 \leq i \leq I$ such that $\bar x_n \in \Gamma_i$. Let $\iota = 3\sqrt{\log(1/\delta)}$ and $\zeta = \iota^2$. Taking union bound, we have
\begin{align}
    &\PP\bigg(S_n>\iota v+ \zeta \bar x_n+ \eta\bigg) \leq \sum_{i=0}^I \PP\bigg(S_n>\iota v+ \zeta \bar x_n+ \eta,\ \bar S_n \leq v,\ \bar x_n \in \Gamma_i\bigg)\notag.
\end{align}
For each $0 \leq i\leq I$, we have
\begin{align}
    &\PP\bigg(S_n>\iota v+ \zeta \bar x_n+ \eta,\ \bar S_n < v,\ \bar x_n \in \Gamma_i\bigg)\notag \\
    & \leq \PP\bigg(S_n > \max\{\iota v, \zeta e^{i-1}\epsilon\},\ \bar S_n < v,\ \bar x_n <e^i\epsilon\bigg)\notag \\
    & \leq \PP\bigg(S_n > \max\{\iota v, \zeta e^{i-1}\epsilon\},\ \bar S_n^2 + \sum_{k=1}^n x_k^2\ind\{|x_k|>e^i\epsilon\}< v^2\bigg)\notag \\
    & \leq \exp\bigg(\frac{-\max\{\iota v, \zeta e^{i-1}\epsilon\}^2}{2(v^2 + e^i\epsilon\cdot\max\{\iota v, \zeta e^{i-1}\epsilon\} /3)}\bigg)\notag \\
    & \leq \exp\bigg(\frac{-1}{2(1/\iota^2+1/\zeta)}\bigg)\notag \\
    & \leq \delta,\label{dz:11}
\end{align}
where the first inequality holds since for any $0 \leq i \leq I$ and $\bar x_n \in \Gamma_i$, we have
\begin{align}
    \iota v+ \zeta \bar x_n+ \eta > \iota v+ \zeta \bar x_n+ \zeta\epsilon/e\geq \max\{\iota v, \zeta e^{i-1}\epsilon\},\notag
\end{align}
the second inequality holds since $\bar x_n <e^i\epsilon$ implies $\ind\{|x_k|>e^j\epsilon\} = 0$ for each $k \in [n]$, the third one holds due to Lemma \ref{lemma:dz}, the fourth one holds by basic calculation, and the last one holds since $\iota = 3\sqrt{\log(1/\delta)}$. Therefore, taking summation for \eqref{dz:11} over all $0 \leq i\leq I$, our statement holds. 
\end{proof}

\begin{lemma}\label{lemma:newbern2}
Let $\{x_i \geq 0\}_{i\geq 1}$ be a stochastic process, $\{\cG_i\}_{i \geq 1}$ be a filtration satisfying $x_i$ is $\cG_i$-measurable. Let $M,v>0$, and $x_i  \leq M$, $\sum_{i=1}^n\EE[x_i|\cG_{i-1}]\leq v$ almost surely. Then for any $\delta, \epsilon>0$, let $\iota = \log(4(\log(M/\epsilon)+2)/\delta)$, we have
\begin{align}
    \PP\bigg(\sum_{i=1}^n x_i \leq 4 \iota v + 11\iota \max_{1 \leq i \leq n} x_i + 4\iota\epsilon\bigg) >1-\delta.\notag
\end{align}
\end{lemma}
\begin{proof}
For simplicity, let
\begin{align}
    S_n = \sum_{i=1}^n x_i,\ \bar S_n = \sum_{i=1}^n\EE[x_i|\cG_{i-1}],\ \bar x_n = \max_{i \in [n]}|x_i|.\notag
\end{align}
We have $\bar S_n \leq v$ and $\bar x_n \leq M$. Set threshold $\epsilon>0$. Set the following sets
\begin{align}
    \Gamma_0 = [0,\epsilon),\ \Gamma_i = [\epsilon\cdot e^{i-1},\epsilon\cdot e^{i}),\ i = 1,\dots, I = \lfloor\log(M/\epsilon)\rfloor+1. \notag
\end{align}
Then there exist $0 \leq i \leq I$ such that $\bar x_n \in \Gamma_i$. Let $\iota = 4\log(4/\delta)$, $\zeta = 11\log(4/\delta)$, $\eta = 4\log(4/\delta)\epsilon$. By union bound we have
\begin{align}
    \PP\bigg(S_n > \iota v + \zeta \bar x_n + \eta\bigg) \leq \sum_{i=0}^I \PP\bigg(S_n>\iota v+ \zeta \bar x_n+ \eta,\ \bar S_n \leq v,\ \bar x_n \in \Gamma_i\bigg)\notag,
\end{align}
For each $0 \leq i \leq I$, we have
\begin{small}
\begin{align}
    &\PP\bigg(S_n>\iota v+ \zeta \bar x_n+ \eta,\ \bar S_n \leq v,\ \bar x_n \in \Gamma_i\bigg)\notag \\
    &\leq \PP\bigg(S_n > \iota v+ \zeta \bar x_n + \eta,\ \bar S_n < \log(4/\delta)(v + e^i\epsilon),\ \bar x_n <e^i\epsilon\bigg)\notag \\
    & = \PP\bigg(\sum_{k=1}^n x_k \ind\{x_k \leq e^i\epsilon\} > \iota v+ \zeta \bar x_n+\eta,\ \bar S_n < \log(4/\delta)(v + e^i\epsilon),\ \bar x_n <e^i\epsilon\bigg)\notag \\
    & \leq \PP\bigg(\sum_{k=1}^n x_k \ind\{x_k \leq e^i\epsilon\} \geq 4\log(4/\delta)(v + e^i\epsilon),\ \bar S_n \leq \log(4/\delta)(v + e^i\epsilon)\bigg)\notag \\
    & \leq \PP\bigg(\sum_{k=1}^n x_k \ind\{x_k \leq e^i\epsilon\} \geq 4\log(4/\delta)(v + e^i\epsilon),\ \sum_{k=1}^n\EE[x_k\ind\{x_k \leq e^i\epsilon\}|\cG_{k-1}] \leq \log(4/\delta)(v + e^i\epsilon)\bigg)\notag \\
    & = \PP\bigg(\sum_{k=1}^n \frac{x_k \ind\{x_k \leq e^i\epsilon\}}{e^i\epsilon} \geq 4\log(4/\delta)\frac{v + e^i\epsilon}{e^i\epsilon},\ \sum_{k=1}^n\EE\bigg[\frac{x_k\ind\{x_k \leq e^i\epsilon\}}{e^i\epsilon}\bigg|\cG_{k-1}\bigg] \leq \log(4/\delta)\frac{v + e^i\epsilon}{e^i\epsilon}\bigg)\notag \\
    & \leq \delta.\label{test:final}
\end{align}
\end{small}
Here, the second inequality holds since for any $0 \leq i \leq I$ and $\bar x_n \in \Gamma_i$, we have
\begin{align}
    \iota v+ \zeta \bar x_n+\eta = 4\log(4/\delta) v+ 11\log(4/\delta) \bar x_n+4\log(4/\delta)\epsilon
    \geq 4\log(4/\delta)(v + e^i\epsilon).\notag
\end{align}
The third inequality holds since $x_k \geq x_k \ind\{x_k \leq e^i\epsilon\}$. The last inequality holds by Lemma \ref{lemma:zhangconcen} with $c = (v+e^i\epsilon)/(e^i\epsilon)>1$. Therefore, taking summation for \eqref{test:final} over all $0 \leq i\leq I$, our statement holds. 
\end{proof}

Now we begin to prove Theorem \ref{lemma:concentration_variance}. We first define the following notations: 
\begin{align}
    \db_1 = \zero,
    \db_k = \sum_{i=1}^{k} \xb_i \eta_i,\ 
    w_k =  \|\xb_k\|_{\Zb^{-1}_{k-1}},\ 
    \event_k = \{1 \le s \leq k, \|\db_s\|_{\Zb_{s}} \leq \beta_s\}\,,\label{eq:concen_main_aux}
\end{align}
where $k\ge 1$ and we define $\beta_1=0$.
Recalling that $x_k$ is $\cG_k$-measurable and $\eta_k$ is $\cG_{k+1}$-measurable,
we find that $\db_k$, $\Zb_k$, $\event_k$ and $w_k$ is $\cG_k$ are $\cG_k$-measurable. Recall that $\beta_t$ is defined as follows:
\begin{align}
    \beta_t &= 12\sqrt{\sigma^2d\log(1+tL^2/(d\lambda))\log(32(\log(R/\epsilon)+1)t^2/\delta)} \notag \\
    &\quad + 24\log(32(\log(R/\epsilon)+1)t^2/\delta)\max_{1 \leq i \leq t} \{|\eta_i|\min\{1, w_i\}\} + 6\log(32(\log(R/\epsilon)+1)t^2/\delta)\epsilon.\notag
\end{align}
The proof of Theorem \ref{lemma:concentration_variance} follows the proof strategy in \citet{dani2008stochastic, zhou2021nearly}. Briefly speaking, we can bound the self-normalized martingale $\|\sum_{i=1}^{k} \xb_i \eta_i\|_{\Zb_k^{-1}}^2$ by the following two terms, 
\begin{align}
    \sum_{i=1}^k \frac{2\eta_i\xb_i^\top \Zb_{i-1}^{-1}\db_{i-1}}{1+w_i^2} \ind\{\event_{i-1}\},\ \text{and}\ \sum_{i=1}^k \frac{\eta_i^2w_i^2}{1+w_i^2}.\notag
\end{align}
Next two lemmas bound these terms by $\beta_t$ separately. The main difference between the following lemmas and their counterparts in \citet{dani2008stochastic, zhou2021nearly} is that our $\beta_t$ is also a \emph{random variable}, which requires us to use some advanced concentration inequalities (Lemmas \ref{lemma:newbern} and \ref{lemma:newbern2}) rather than vanilla ones.

\begin{lemma}\label{lemma:martingale_first}
Let $\db_i, w_i, \event_i$ be defined in \eqref{eq:concen_main_aux}.
Then, with probability at least $1-\delta/2$, 
simultaneously for all $k\ge 1$ it holds that
\begin{align}
\sum_{i=1}^k \frac{2\eta_i\xb_i^\top \Zb_{i-1}^{-1}\db_{i-1}}{1+w_i^2} \ind\{\event_{i-1}\} \leq 3\beta_{k}^2/4.\notag
\end{align}
\end{lemma}
\begin{proof}
We have
\begin{align}
    \bigg|\frac{2\xb_i^\top \Zb_{i-1}^{-1}\db_{i-1}}{1+w_i^2} \ind\{\event_{i-1}\}\bigg| \leq \frac{2\|\xb_i\|_{\Zb_{i-1}^{-1}} [\|\db_{i-1}\|_{\Zb_{i-1}^{-1}}\ind\{\event_{i-1}\}]}{1+w_i^2} \leq \frac{2w_i\beta_{i-1}}{1+w_i^2} \leq \min\{1, 2w_i\}\beta_{i-1}, \label{eq:martingale_first_0}
\end{align}
where the first inequality holds due to Cauchy-Schwarz inequality, the second inequality holds due to the definition of $\event_{i-1}$, the last inequality holds by algebra. Fix $t>0$. For simplicity, let $\ell_i$ denote
\begin{align}
    \ell_i =  \frac{2\eta_i\xb_i^\top \Zb_{i-1}^{-1}\db_{i-1}}{\beta_{i-1}(1+w_i^2)} \ind\{\event_{i-1}\}.
\end{align}
We are preparing to apply Lemma \ref{lemma:newbern} to $\{\ell_i\}_i$ and $\{\cG_i\}_i$.
First note that 
$\EE[\ell_i|\cG_i]=0$. Meanwhile, by \eqref{eq:martingale_first_0} we have $|\ell_i| \leq R$ and 
\begin{align}
    |\ell_i| \leq  2|\eta_i|\min\{1, w_i\}.\label{eq:martingale_first_1}
\end{align}
We also have
\begin{align}
    \sum_{i=1}^t \EE[\ell_i^2|\cG_i] &\leq \sigma^2 \sum_{i=1}^t\bigg(\frac{2\xb_i^\top \Zb_{i-1}^{-1}\db_{i-1}}{1+w_i^2} \ind\{\event_{i-1}\}/\beta_{i-1}\bigg)^2\notag \\
    & \leq 4\sigma^2\sum_{i=1}^t [\min\{1, w_i\}]^2\notag \\
    & \leq 4\sigma^2\sum_{i=1}^t\min\{1, w_i^2\}\notag \\
    & \leq 8\sigma^2d\log(1+tL^2/(d\lambda)),\label{eq:martingale_first_2}
\end{align}
where the first inequality holds since $\EE[\eta_i^2|\cG_i] \leq \sigma^2$, the second inequality holds due to \eqref{eq:martingale_first_0}, the third inequality holds again since $\{\beta_i\}_i$ is increasing, 
the last inequality holds due to Lemma \ref{lemma:sumcontext}.
Let $\iota_t = \sqrt{\log(4(\log(R/\epsilon)+2)t^2/\delta)}$.
Therefore, by 
\eqref{eq:martingale_first_1} and \eqref{eq:martingale_first_2}, using Lemma \ref{lemma:newbern}, we know that with probability at least $1-\delta/(4t^2)$, we have
\begin{align}
    \sum_{i=1}^t \ell_i &\leq 3\iota_t\sqrt{8\sigma^2d\log(1+tL^2/(d\lambda))} + 18\iota_t^2 \max_{1 \leq i \leq t} \{|\eta_i|\min\{1, w_i\}\} + 4\iota_t^2\epsilon \leq \frac{3}{4}\beta_t,\notag
\end{align}
where the last inequality holds due to the definition of $\beta_t$. 
Therefore, using the fact $\beta_{i-1} \leq \beta_t$, we have
\begin{align}
 \sum_{i=1}^t \frac{2\eta_i\xb_i^\top \Zb_{i-1}^{-1}\db_{i-1}}{(1+w_i^2)} \ind\{\event_{i-1}\} \leq \beta_t \sum_{i=1}^t\ell_i \leq  3\beta_t^2/4. \label{eq:martingale_first_3}
\end{align}
Taking union bound for \eqref{eq:martingale_first_3} from $t = 1$ to $\infty$ and using the fact that $\sum_{t=1}^\infty t^{-2} <2$ finishes the proof.
\end{proof}

We also need the following lemma.

\begin{lemma}\label{lemma:martingale_second}
Let $w_i$ be defined in \eqref{eq:concen_main_aux}.
Then, with probability at least $1-\delta/2$, simultaneously for all $k\ge 1$ it holds that
\begin{align}
\sum_{i=1}^k \frac{\eta_i^2w_i^2}{1+w_i^2} \leq \beta_{k}^2/4.\notag
\end{align}
\end{lemma}
\begin{proof}
For simplicity, let
\begin{align}
    \ell_i = \frac{\eta_i^2w_i^2}{1+w_i^2}.\notag
\end{align}
Fix $t$. Then by the fact that $\EE[\eta_i^2|\cG_i] \leq \sigma^2$ and Lemma \ref{lemma:sumcontext}, we have
\begin{align}
    \sum_{i=1}^t \EE[\ell_i|\cG_i] = \sum_{i=1}^t \EE\bigg[\frac{\eta_i^2w_i^2}{1+w_i^2}\bigg|\cG_i\bigg] \leq \sigma^2  \sum_{i=1}^t \min\{1, w_i^2\} \leq 2\sigma^2d\log(1+tL^2/(d\lambda)).
\end{align}
Furthermore, we have $|\ell_i| \leq R^2$ and
\begin{align}
    \ell_i \leq|\eta_i|^2\min\{1, w_i\}^2 .\notag
\end{align}
Let $\iota_t = \log(32(\log(R/\epsilon)+1)t^2/\delta)$. Then by Lemma \ref{lemma:newbern2}, we have with probability at least $1-\delta/(4t^2)$, 
\begin{align}
    \sum_{i=1}^t \ell_i \leq 8\iota_t\sigma^2d\log(1+tL^2/(d\lambda))   + 11\iota_t \max_{1 \leq i \leq t} \{|\eta_i|\min\{1, w_i\}\}^2 + 4\iota_t\epsilon^2 \leq \beta_t^2/4.\label{eq:8888_2}
\end{align}
where the last inequality holds due to the definition of $\beta_t$. Taking union bound for \eqref{eq:8888_2} from $t = 1$ to $\infty$ and using the fact that $\sum_{t=1}^\infty t^{-2} <2$ 
finishes the proof.
\end{proof}

With above two lemmas, we are ready to prove
Theorem \ref{lemma:concentration_variance}.
\begin{proof}[Proof of Theorem \ref{lemma:concentration_variance}]
The proof is nearly the same as the proof of Theorem 2 in \citet{zhou2021nearly}, only to replace Lemmas 13 and 14 in \citet{zhou2021nearly} with Lemma \ref{lemma:martingale_first} and \ref{lemma:martingale_second}. 
\end{proof}

\section{Proof of Main Results in Section \ref{sec:linearbandit}}\label{app:concen_main}
In this section, we prove Lemma \ref{lemma:keysum:temp} and Theorem \ref{coro:linearregret}.

\subsection{Proof of Lemma \ref{lemma:keysum:temp}}

We restate Lemma \ref{lemma:keysum:temp} here with a different set of notations, as it will be used in the later proof for RL as well.
\begin{lemma}[Restatement of Lemma \ref{lemma:keysum:temp}]\label{lemma:keysum}
Let $\{\sigma_k, \beta_k\}_{k \geq 1}$ be a sequence of non-negative numbers, $\alpha, \gamma>0$, $\{\xb_k\}_{k \geq 1} \subset \RR^d$ and $\|\xb_k\|_2 \leq L$. Let $\{\Zb_k\}_{k \geq 1}$ and $\{\bar\sigma_k\}_{k \geq 1}$ be recursively defined as follows: $\Zb_1 = \lambda\Ib$,\ 
\begin{align}
    \forall k \geq 1,\ \bar\sigma_k = \max\{\sigma_k, \alpha, \gamma\|\xb_k\|_{\Zb_k^{-1}}^{1/2}\},\ \Zb_{k+1} = \Zb_k + \xb_k\xb_k^\top/\bar\sigma_k^2.\notag
\end{align}
Let $\iota = \log(1+KL^2/(d\lambda\alpha^2))$. Then we have
\begin{align}
    \sum_{k=1}^K\min\Big\{1, \beta_k\|\xb_k\|_{\Zb_k^{-1}}\Big\} &\leq 2d \iota +2\max_{k \in [K]}\beta_k \gamma^2d\iota+ 2\sqrt{d \iota}\sqrt{\sum_{k=1}^K\beta_k^2(\sigma_k^2 + \alpha^2)}\notag .
\end{align}
\end{lemma}

\begin{proof}[Proof of Lemma \ref{lemma:keysum}]
We decompose the set $[K]$ into a union of two disjoint subsets $[K] = \cI_1 \cup \cI_2$. 
\begin{align}
    \cI_1 = \Big\{k \in [K]:\|\xb_k/\bar\sigma_k\|_{\Zb_k^{-1}} \geq 1 \Big\},\ \cI_2 = [K]\setminus \cI_1.
\end{align}
Then the following upper bound of $|\cI_1|$ holds:
\begin{align}
    |\cI_1| = \sum_{k\in\cI_1}\min\Big\{1, \|\xb_k/\bar\sigma_k\|_{\Zb_k^{-1}}^2\Big\} \leq \sum_{k=1}^K\min\Big\{1, \|\xb_k/\bar\sigma_k\|_{\Zb_k^{-1}}^2\Big\} \leq 2d \iota,\label{eq:linearregret_12}
\end{align}
where the first inequality holds since $\|\xb_k/\bar\sigma_k\|_{\Zb_k^{-1}} \geq 1$ for $k \in \cI_1$, the third inequality holds due to Lemma \ref{lemma:sumcontext} together with the fact $\|\xb_k/\bar\sigma_k\|_2 \leq L/\alpha$. Therefore, we have
\begin{align}
    &\sum_{k\in [K]}\min\Big\{1, \beta_k\|\xb_k\|_{\Zb_k^{-1}}\Big\}\notag \\
    & =\sum_{k\in \cI_1}\min\Big\{1, \bar\sigma_k\beta_k\|\xb_k/\bar\sigma_k\|_{\Zb_k^{-1}}\Big\} + \sum_{k\in \cI_2}\min\Big\{1, \bar\sigma_k\beta_k\|\xb_k/\bar\sigma_k\|_{\Zb_k^{-1}}\Big\}\notag \\
    & \leq \bigg[\sum_{k\in \cI_1} 1\bigg] + \sum_{k\in \cI_2}\beta_k\bar\sigma_k\|\xb_k/\bar\sigma_k\|_{\Zb_k^{-1}}\notag \\
    & \leq 2d \iota + \sum_{k\in\cI_2}\beta_k\bar\sigma_k \|\xb_k/\bar\sigma_k\|_{\Zb_k^{-1}},\label{eq:linearregret_13}
\end{align}
where the first inequality holds since  $\min\{1,x\}\le 1$ and also $\min\{1,x\}\le x$, 
the second inequality holds due to \eqref{eq:linearregret_12}. Next we further bound the second summation term in \eqref{eq:linearregret_13}. We decompose $\cI_2 = \cJ_1 \cup \cJ_2$, where
\begin{align}
    &\cJ_1 = \bigg\{k \in \cI_2:\bar\sigma_k = \sigma_k \cup\bar\sigma_k = \alpha \bigg\},\ \cJ_2 = \bigg\{k \in \cI_2:\bar\sigma_k = \gamma\sqrt{\|\xb_k\|_{\Zb_k^{-1}}}\bigg\}.\notag
\end{align}
For the summation over $\cJ_1$, we have
\begin{align}
    \sum_{k\in\cJ_1}\beta_k\bar\sigma_k \|\xb_k/\bar\sigma_k\|_{\Zb_k^{-1}}
    & \leq  \sum_{k\in\cJ_1}\beta_k(\sigma_k + \alpha) \min\{1, \|\xb_k/\bar\sigma_k\|_{\Zb_k^{-1}}\}\notag \\
    & \leq  \sum_{k=1}^K\beta_k(\sigma_k + \alpha) \min\{1, \|\xb_k/\bar\sigma_k\|_{\Zb_k^{-1}}\}\notag \\
    & \leq  \sqrt{2\sum_{k=1}^K(\sigma_k^2 + \alpha^2)\beta_k^2}\sqrt{\sum_{k=1}^K\min\{1, \|\xb_k/\bar\sigma_k\|_{\Zb_k^{-1}}\}^2}\notag \\
    & \leq 2\sqrt{\sum_{k=1}^K\beta_k^2(\sigma_k^2 + \alpha^2)}\sqrt{d \iota},\label{eq:linearregret_13.1}
\end{align}
where the first inequality holds since $\bar\sigma_k \leq \sigma_k + \alpha$ for $k \in \cJ_1$ and $\|\xb_k/\bar\sigma_k\|_{\Zb_k^{-1}} \leq 1$ since $k \in \cJ_1 \subseteq \cI_2$, the third one holds due to Cauchy-Schwarz inequality and the last one holds due to Lemma \ref{lemma:sumcontext}. Next we bound the summation over $\cJ_2$. First, for $k \in \cJ_2$, we have the following equation:
\begin{align}
    \bar\sigma_k = \gamma^2  \|\xb_k/\bar\sigma_k\|_{\Zb_k^{-1}}.\notag
\end{align}
Then we can bound the summation over $\cJ_2$ as follows:
\begin{small}
\begin{align}
    \sum_{k\in\cJ_2}\beta_k\bar\sigma_k \|\xb_k/\bar\sigma_k\|_{\Zb_k^{-1}}
    & = \gamma^2\cdot \sum_{k\in\cJ_1}\beta_k \|\xb_k/\bar\sigma_k\|^2_{\Zb_k^{-1}} = \gamma^2\cdot \sum_{k=1}^K\beta_k \min\{1,\|\xb_k/\bar\sigma_k\|^2_{\Zb_k^{-1}}\}\leq 2\max_{k\in[K]}\beta_k \gamma^2d\iota,\label{eq:linearregret_13.2}
\end{align}
\end{small}
where the inequality holds due to Lemma \ref{lemma:sumcontext}.  Substituting \eqref{eq:linearregret_13.1} and \eqref{eq:linearregret_13.2} into \eqref{eq:linearregret_13} completes the proof.
\end{proof}

\subsection{Proof of Theorem \ref{coro:linearregret}}\label{sec:proof:linearregret}

\begin{proof}[Proof of Theorem \ref{coro:linearregret}]

By the assumption on $\epsilon_k$, we know that
\begin{align}
    &|\epsilon_k/\bar\sigma_k| \leq R/\alpha,\notag\\
    &|\epsilon_k/\bar\sigma_k|\cdot\min\{1,\|\ab_k/\bar\sigma_k\|_{\hat\bSigma_k^{-1}}\} \leq R\|\ab_k\|_{\hat\bSigma_k^{-1}}/\bar\sigma_k^2 \leq R/\gamma^2,\notag \\
    &\EE[\epsilon_k|\ab_{1:k}, \epsilon_{1:k-1}] = 0,\ \EE [(\epsilon_k/\bar\sigma_k)^2|\ab_{1:k}, \epsilon_{1:k-1}] \leq 1,\ \|\ab_k/\bar\sigma_k\|_2 \leq A/\alpha,\notag
\end{align}
Therefore, setting $\cG_k = \sigma(\ab_{1:k}, \epsilon_{1:k-1})$, 
and using that $\sigma_k$ is $\cG_k$-measurable, applying
Theorem \ref{lemma:concentration_variance} to $(\bx_k,\eta_k)=(\ba_k/\bar\sigma_k,\epsilon_k/\bar\sigma_k)$ with $\epsilon = R/\gamma^2$, we get that
with probability at least $1-\delta$, for all $k \geq 1$, 
\begin{align}
    \big\|\hat\btheta_k - \btheta^*\big\|_{\hat\bSigma_k}
    & \leq 12\sqrt{d\log(1+kA^2/(\alpha^2d\lambda))\log(32(\log(\gamma^2/\alpha)+1)k^2/\delta)} \notag \\
    &\quad + 30\log(32(\log(\gamma^2/\alpha)+1)k^2/\delta)R/\gamma^2 + \sqrt{\lambda}B\notag \\
    &: = \hat\beta_k,\label{eq:helper}
\end{align}
Let $\event_{\ref{eq:helper}}$ denote the event such that \eqref{eq:helper} happens for all $k \geq 1$. 
Then, on the event $\event_{\ref{eq:helper}}$, we have 
\begin{align}
    \la \ab_k^*, \btheta^*\ra \leq \la \ab_k^*, \hat\btheta_k\ra + \|\ab_k^*\|_{\hat\bSigma_k^{-1}}\|\hat\btheta_k - \btheta^*\|_{\hat\bSigma_k} \leq \la \ab_k^*, \hat\btheta_k\ra + \hat\beta_k\|\ab_k^*\|_{\hat\bSigma_k^{-1}} \leq \la \ab_k, \hat\btheta_k\ra + \hat\beta_k\|\ab_k\|_{\hat\bSigma_k^{-1}},\label{eq:linearregret_0}
\end{align}
where the last inequality holds due to the selection of $\ab_k$. By \eqref{eq:linearregret_0}, we have
\begin{align}
    \la \ab_k^*, \btheta^*\ra - \la \ab_k, \btheta^*\ra\leq \|\ab_k\|_{\hat\bSigma_k^{-1}}(\hat\beta_k + \|\btheta^* - \hat\btheta_k\|_{\hat\bSigma_k}) \leq 2\hat\beta_k\|\ab_k\|_{\hat\bSigma_k^{-1}},\label{eq:linearregret_0.1}
\end{align}
where the first inequality holds due to Cauchy-Schwarz inequality, the second one holds due to event $\event_{\ref{eq:helper}}$. Meanwhile, we have $0 \leq \la \ab_k^*, \btheta^*\ra - \la \ab_k, \btheta^*\ra \leq 2$. Thus, substituting \eqref{eq:linearregret_0.1} into \eqref{eq:linearregret_0} and summing up \eqref{eq:linearregret_0} for $k\in[K]$, we have
\begin{align}
    \text{Regret}(K) = \sum_{k=1}^K\big[\la \ab_k^*, \btheta^*\ra - \la \ab_k, \btheta^*\ra\big]
    &\leq 2\sum_{k=1}^K\min\Big\{1, \hat\beta_K\|\ab_k\|_{\hat\bSigma_k^{-1}}\Big\}.\label{eq:linearregret_11}
\end{align}
Finally, to bound \eqref{eq:linearregret_11}, by Lemma \ref{lemma:keysum} we have
\begin{align}
    \text{Regret}(K)/2 &\leq 2d \iota + 2d\gamma^2\hat\beta_K\iota + \hat\beta_K\sqrt{2\sum_{k=1}^K\sigma_k^2 + 2K\alpha^2}\sqrt{2d \iota},\notag
\end{align}
which completes our proof.
\end{proof}

\section{Proof of Main Results in Section \ref{section:finite_main}}\label{sec:app_finite_1}

For $k \in [K]$, $h \in [H]$, let $\cF_{k,h}$ be the $\sigma$-algebra generated by the random variables representing the state-action pairs up to and including those that appear stage $h$ of episode $k$. That is, $\cF_{k,h}$ is generated by
\begin{align*}
s_1^1,a_1^1, \dots, s_h^1,a_h^1, &\dots, s_H^1,a_H^1\,, \\
s_1^2,a_1^2, \dots, s_h^2,a_h^2, &\dots, s_H^2,a_H^2\,, \\
\vdots \\
s_1^k,a_1^k,\dots, s_h^k,a_h^k & \,.
\end{align*}
Note that, by construction,
\begin{align*}
 \bar\var_{k,h}\vvalue_{k,h+1}^{2^m}(s_h^k, a_h^k),
 E_{k,h,m},
 \bar \sigma_{k,h,m},
\tilde\bSigma_{k,h+1,m},
\bphi_{k,h,m}
\end{align*}
are $\cF_{k,h}$-measurable, and $Q_{k,h},V_{k,h}, \pi_h^k, \hat\bSigma_{k,m}, \hat\bb_{k,m},
\hat\btheta_{k,m}$ are $\cF_{k-1,H}$ measurable.

First we propose a lemma that suggests that the vector $\btheta^*$ lies in a series of confidence sets, which further implies that the difference between the estimated high-order moment and true moment can be bounded by error terms $E_{k,h,m}$. 
\begin{lemma}\label{thm:concentrate:finite}
Set $\{\hat\beta_k\}_{k \geq 1}$ as \eqref{def:hatbeta}, then, with probability at least $1-M\delta$, we have for any $k \in [K], h \in[H],m \in \seq{M}$,
\begin{align}
\big\|\hat\bSigma_{k,m}^{1/2}\big(\hat\btheta_{k,m} - \btheta^*\big)\big\|_2 \leq \hat\beta_k,\ |[\bar\var_{k,m}\vvalue_{k,h+1}^{2^m}](s_h^k, a_h^k) - [\var\vvalue_{k,h+1}^{2^m}](s_h^k, a_h^k)| \leq \error_{k,h,m}.\notag
\end{align}
\end{lemma}

\subsection{Proof of Lemma \ref{thm:concentrate:finite}}\label{sec:proof:concentrate:finite}
We start with the following lemma. 
\begin{lemma}\label{lemma:variancebound:finite}
Let $V_{k,h}, \hat\bSigma_{k,m}, \hat\btheta_{k,m}, \bphi_{k,h,m}, \bar\var_{k,m}$ be defined in Algorithm \ref{algorithm:finite}. For any $k \in [K], h \in [H], m \in \seq{M}$, we have
\begin{align}
    &\big|\var\vvalue_{k,h+1}^{2^m}(s_h^k, a_h^k) - \bar\var_{k,m}\vvalue_{k,h+1}^{2^m}(s_h^k, a_h^k)\big|\notag \\
    &\leq   \min\Big\{1, \Big\|\hat\bSigma_{k,m+1}^{-1/2}\bphi_{k,h,m+1}\Big\|_2 \Big\|\hat\bSigma_{k,m+1}^{1/2}\big(\hat\btheta_{k,m+1} - \btheta^*\big)\Big\|_2\Big\}\notag \\
    &\qquad + \min\Big\{1,2\Big\|\hat\bSigma_{k,m}^{-1/2}\bphi_{k,h,m}\Big\|_2 \Big\|\hat\bSigma_{k,m}^{1/2}\big(\hat\btheta_{k,m} - \btheta^*\big)\Big\|_2\Big\}.\notag
\end{align}

\end{lemma}
\begin{proof}
The proof is nearly identical to the proof of Lemma 15 in \citet{zhou2021nearly}. The only difference is to replace the upper bound $H$ with 1, $\vvalue_{k,h+1}$ with $\vvalue_{k,h+1}^{2^m}$, $\tilde\bSigma_{k,h}$ with $\hat\bSigma_{k,m+1}$, $\tilde\btheta_{k,h}$ with $\hat\btheta_{k,m+1}$, $\hat\bSigma_{k,h}$ with $\hat\bSigma_{k,m}$, $\hat\btheta_{k,h}$ with $\hat\btheta_{k,m}$, and $\btheta_h^*$ with $\btheta^*$.
\end{proof}

\begin{proof}[Proof of Lemma \ref{thm:concentrate:finite}]
First we recall the definition of $\bar\sigma_{k,h,m}$:
\begin{align}
		    &\bar\sigma_{k,h,m}^2 = \max\big\{ [\bar\var_{k,m}\vvalue_{k, h+1}^{2^m}](s_h^k, a_h^k) + \error_{k,h,m}, \alpha^2, \gamma^2\Big\|\hat\bSigma_{k,m}^{-1/2}\bphi_{k,h,m}\Big\|_2\big\},\notag \\
		    & \bar\sigma_{k,h,M-1}^2 = \max\big\{ 1, \alpha^2, \gamma^2\Big\|\hat\bSigma_{k,M-1}^{-1/2}\bphi_{k,h,M-1}\Big\|_2\big\},\notag
\end{align}
We prove the statement by induction. We define the following quantities. For simplicity, let $\hat\cC_{k,m}$ be defined as
\begin{align}
    \hat\cC_{k,m} := \Big\{\btheta:\Big\|\hat\bSigma_{k,m}^{1/2}\big(\hat\btheta_{k,m} - \btheta\big)\Big\|_2 \leq \hat\beta_k\Big\}.\notag
\end{align}
For each $m$, let 
\begin{align}
    &\xb_{k,h,m} = \bar\sigma_{k,h, m}^{-1}\bphi_{k,h,m},\notag \\
    & \eta_{k,h,m} = \bar\sigma_{k,h,m}^{-1}\ind\{\btheta^* \in \hat\cC_{k,m}\cap \hat\cC_{k,m+1}\}\big[\vvalue_{k, h+1}^{2^m}(s_{h+1}^k) - \la\bphi_{k,h,m}, \btheta^* \ra\big],\notag \\
    & \eta_{k,h,M-1} = \bar\sigma_{k,h,M-1}^{-1}\big[\vvalue_{k, h+1}^{2^{M-1}}(s_{h+1}^k) - \la\bphi_{k,h,M-1}, \btheta^* \ra\big]\notag \\
    & \cG_{k,h} = \cF_{k,h}, \notag \\
    & \bmu^* = \btheta^*.\notag
\end{align}
We have 
\begin{align}
    &\EE[\eta_{k,h,m}|\cG_{k,h}] = 0,\ \|\xb_{k,h,m}\|_2 \leq \bar\sigma_{k,h,m}^{-1} \leq 1/\alpha,\ |\eta_{k,h,m}| \leq 1/\alpha.\notag
\end{align}
Note that $\ind\{\btheta^* \in \hat\cC_{k,m}\cap \hat\cC_{k,m+1}\}$ is $\cG_{k,h}$-measurable, then we have bound the variance of $\eta_{k,h,m}$ as follows: for $m \in \seq{M-1}$,
\begin{align}
    \EE[\eta_{k,h,m}^2|\cG_{k,h}] &= \bar\sigma_{k,h,m}^{-2}\ind\{\btheta^* \in \hat\cC_{k,m}\cap \hat\cC_{k,m+1}\}[\var\vvalue_{k, h+1}^{2^m}](s_h^k, a_h^k)\notag \\
    & \leq \bar\sigma_{k,h,m}^{-2}\ind\{\btheta^* \in \hat\cC_{k,m}\cap \hat\cC_{k,m+1}\}\bigg[\bar\var_{k,m}\vvalue_{k,h+1}^{2^m}(s_h^k, a_h^k)\notag \\
    &\quad + \min\Big\{1, \Big\|\hat\bSigma_{k,m+1}^{-1/2}\bphi_{k,h,m+1}\Big\|_2 \Big\|\hat\bSigma_{k,m+1}^{1/2}\big(\hat\btheta_{k,m+1} - \btheta^*\big)\Big\|_2\Big\}\notag \\
    &\quad + \min\Big\{1,2\Big\|\hat\bSigma_{k,m}^{-1/2}\bphi_{k,h,m}\Big\|_2 \Big\|\hat\bSigma_{k,m}^{1/2}\big(\hat\btheta_{k,m} - \btheta^*\big)\Big\|_2\Big\}\bigg]\notag \\
    & \leq \bar\sigma_{k,h,m}^{-2}\bigg[\bar\var_{k,m}\vvalue_{k,h+1}^{2^m}(s_h^k, a_h^k) + \min\Big\{1, \hat\beta_k\Big\|\hat\bSigma_{k,m+1}^{-1/2}\bphi_{k,h,m+1}\Big\|_2 \Big\}\notag \\
    &\quad + \min\Big\{1,2\hat\beta_k\Big\|\hat\bSigma_{k,m}^{-1/2}\bphi_{k,h,m}\Big\|_2 \Big\}\bigg]\notag \\
    & \leq 1.\notag
\end{align}
For $m = M-1$, we directly have $\EE[\eta_{k,h,M-1}^2|\cG_{k,h}] \leq 1$. Meanwhile, for any $m \in \seq{M}$, we have
\begin{align}
    |\eta_{k,h,m}|\max\{1, \|\xb_{k,h,m}\|_{\tilde\bSigma_{k,h-1,m}^{-1}}\} \leq  \bar\sigma_{k,h,m}^{-2}\|\bphi_{k,h,m}\|_{\tilde\bSigma_{k,h-1,m}^{-1}} \leq 1/\gamma^2.\notag
\end{align}

Now, let $y, \Zb, \bbb, \bmu, \epsilon$ defined in Theorem \ref{lemma:concentration_variance} be set as follows:
\begin{align}
    &y_{k,h,m} = \la \bmu^*,\xb_{k,h,m}\ra + \eta_{k,h,m}\notag \\
    &\Zb_{k,m} = \lambda \Ib + \sum_{i = 1}^k\sum_{h=1}^H \xb_{i,h,m}\xb_{i,h,m}^\top = \hat\bSigma_{k+1, m}\notag \\
    &\bbb_{k,m}  = \sum_{i = 1}^k \sum_{h=1}^H\xb_{i,h,m}y_{i,h,m}\notag \\
    &\bmu_{k,m} = \Zb_{k,m}^{-1}\bbb_{k,m},\notag \\
    &\epsilon = 1/\gamma^2.\notag
\end{align}
Then, by Theorem \ref{lemma:concentration_variance}, for each $m \in \seq{M}$, with probability at least $1-\delta$, $\forall k \in [K+1]$, 
\begin{align}
    \|\bmu_{k-1,m} - \btheta^*\|_{\hat\bSigma_{k,m}}
    & \leq 12\sqrt{d\log(1+kH/(\alpha^2d\lambda))\log(32(\log(\gamma^2/\alpha)+1)k^2H^2/\delta)} \notag \\
    &\quad + 30\log(32(\log(\gamma^2/\alpha)+1)k^2H^2/\delta)/\gamma^2 + \sqrt{\lambda}\pnorm\notag \\
    & = \hat\beta_{k}.\label{eq:concentrate:finite:2}
\end{align}
Denote the event that \eqref{eq:concentrate:finite:2} happens for all $k$ and $m$ as $\event$. Conditioned on $\event$, we have the following observations:
\begin{itemize}[leftmargin = *]
    \item For $k = 1$, $m \in \seq{M}$, by the definitions of $\hat\btheta_{1,m}, \hat\bSigma_{1,m}$, we have $\|\btheta^* - \hat\btheta_{1,m}\|_{\hat\bSigma_{1,m}} = \|\btheta^*\|_{\lambda \Ib} \leq \sqrt{\lambda}B = \hat\beta_1$, which implies
    \begin{align}
        \btheta^* \in \hat\cC_{1,m}.\label{test1}
    \end{align}
    \item For $k \in [K]$ and $m = M-1$, we directly have $\bmu_{k,M-1} = \hat\btheta_{k+1, M-1}$, which implies
\begin{align}
     \btheta^* \in \hat\cC_{k+1, M-1}.\label{test2}
\end{align}
\item For $k \in [K]$ and $m \in \seq{M-1}$, we have
\begin{align}
    \btheta^* \in \hat\cC_{k,m}\cap \hat\cC_{k,m+1} \Rightarrow y_{k,h,m} = \bar\sigma_{k,h,m}^{-1}\vvalue_{k, h+1}^{2^m}(s_{h+1}^k)\Rightarrow \bmu_{k,m} = \hat\btheta_{k+1, m}\Rightarrow \btheta^* \in \hat\cC_{k+1, m}.\label{test3}
\end{align}
\end{itemize}
Therefore by induction based on the initial conditions \eqref{test1} and \eqref{test2} and induction rule \eqref{test3}, we have for $k \in [K]$ and $m \in \seq{M}$, $\btheta^* \in \hat\cC_{k,m}$. Lastly, conditioned on $\event$, by Lemma \ref{thm:concentrate:finite}, we have for all $k \in [K], h \in [H], m \in \seq{M}$, 
\begin{align}
    \big|[\bar\var_{k,m}\vvalue_{k,h+1}^{2^m}](s_h^k, a_h^k) - [\var\vvalue_{k,h+1}^{2^m}](s_h^k, a_h^k)\big| \leq \error_{k,h,m}.\notag
\end{align}
\end{proof}

\subsection{Proof of Theorem \ref{thm:regret:finite}}\label{sec:proof:regret:finite}
Let $\event_{\ref{thm:concentrate:finite}}$ denote the event described by Lemma \ref{thm:concentrate:finite}. 
We have the following lemmas.

\begin{lemma}\label{lemma:upper:finite}
Let $\qvalue_{k,h}, \vvalue_{k,h}$ be defined in Algorithm \ref{algorithm:finite}. 
Then, on the event $\event_{\ref{thm:concentrate:finite}}$, 
for any $s,a, k, h$ we have that
$\qvalue_h^*(s,a) \leq \qvalue_{k,h}(s,a)$, $\vvalue_h^*(s) \leq \vvalue_{k,h}(s)$.
\end{lemma}
\begin{proof}
The proof is nearly identical to the proof of Lemma 19 in \citealt{zhou2021nearly}. The only difference is to replace the event used in Lemma 19 in \citet{zhou2021nearly} with the event defined by Lemma \ref{thm:concentrate:finite}. Besides, we replace $\hat\cC_{k,h}$ with $\hat\cC_{k,0}$, $\hat\btheta_{k,h}$ with $\hat\btheta_{k,0}$, $\hat\bSigma_{k,h}$ with $\hat\bSigma_{k,0}$, $\PP_h$ with $\PP$ and $\btheta_h^*$ with $\btheta$.
\end{proof}

\begin{lemma}\label{lemma:expectation}
Lemma $V_{k,h}, \hat\bSigma_{k,0}, \bphi_{k,h,0}$ be defined in Algorithm \ref{algorithm:finite}, $\hat\beta_k$ be defined in \eqref{def:hatbeta}. Then on the event $\event_{\ref{thm:concentrate:finite}}$, for any $k \in [K], h \in [H]$, we have
\begin{align}
    V_{k,h}(s_h^k) - r(s_h^k, a_h^k) - [\PP V_{k,h+1}](s_h^k, a_h^k) \leq 2\min\{1, \hat\beta_k\Big\|\hat \bSigma_{k,0}^{-1/2} \bphi_{k,h,0}\Big\|_2\}
\end{align}
\end{lemma}
\begin{proof}
For any $k \in [K]$ and $h \in [H]$, we have
\begin{align}
    &V_{k,h}(s_h^k) - r(s_h^k, a_h^k) - [\PP V_{k,h+1}](s_h^k, a_h^k) \notag \\
    & \leq \la\bphi_{V_{k,h+1}}(s_h^k, a_h^k), \hat\btheta_{k,0} \ra + \hat\beta_k\Big\|\hat \bSigma_{k,0}^{-1/2} \bphi_{\vvalue_{k, h+1}}(s_h^k, a_h^k)\Big\|_2 - \la\bphi_{V_{k,h+1}}(s_h^k, a_h^k), \btheta^* \ra\notag \\
    & \leq \Big\|\hat \bSigma_{k, 0}^{1/2} (\hat\btheta_{k,0} -\btheta^* )\Big\|_2\Big\|\hat \bSigma_{k,0}^{-1/2} \bphi_{k,h,0}\Big\|_2 + \hat\beta_k\Big\|\hat \bSigma_{k,0}^{-1/2} \bphi_{k,h,0}\Big\|_2\notag \\
    & \leq 2\hat\beta_k\Big\|\hat \bSigma_{k,0}^{-1/2} \bphi_{k,h,0}\Big\|_2,
\end{align}
where the first inequality holds due to the definition of $V_{k,h}$, the second one holds due to Cauchy-Schwarz inequality and the last one holds since on $\event_{\ref{thm:concentrate:finite}}$, $\hat\btheta_{k,0} \in \hat\cC_{k,0}$. Meanwhile, we have $V_{k,h}(s_h^k) - r(s_h^k, a_h^k) - [\PP V_{k,h+1}](s_h^k, a_h^k) \leq 2$. Therefore, combining two bounds completes the proof. 
\end{proof}

Recall $\iota = \log(1+KH/(d\lambda\alpha^2)),\ \zeta = 4\log(4\log(KH)/\delta)$. For any $k \in [K], h \in [H]$ we define the following indicator function $I_h^k$ such that
\begin{align}
    I_h^k: = \ind\Big\{\forall m \in \seq{M},\ \det(\hat\bSigma_{k,m}^{-1/2})/\det(\tilde\bSigma_{k,h,m}^{-1/2}) \leq 4\Big\}.\notag
\end{align}
Clearly $I_h^k$ is $\cF_{k,h}$-measurable and monotonically decreasing. For $m \in \seq{M}$, we also define the following quantities:
\begin{align}
    &R_m = \sum_{k=1}^K \sum_{h=1}^H I_h^k\min\Big\{1, \hat\beta_k\Big\|\hat \bSigma_{k,m}^{-1/2} \bphi_{k,h,m}\Big\|_2\Big\},\label{def:rm}\\
    & A_m = \sum_{k=1}^K \sum_{h=1}^H I_h^k[[\PP V_{k,h+1}^{2^m}](s_h^k, a_h^k) - V_{k,h+1}^{2^m}(s_{h+1}^k)]\label{def:am},\\
    & S_m = \sum_{k=1}^K\sum_{h=1}^HI_h^k[\var\vvalue_{k, h+1}^{2^m}](s_h^k, a_h^k)\label{def:sm},\\
    &G = \sum_{k=1}^K(1-I_H^k).\label{def:g}
\end{align}

We aim to bound $R_m, A_m, S_m, G$ by the following lemmas.

\begin{lemma}\label{lemma:rm}
Let $\gamma, \alpha$ be defined in Algorithm \ref{algorithm:finite}, $\{R_m, S_m\}_{m \in \seq{M}}$ be defined in \eqref{def:rm} and \eqref{def:sm}. Then for $m \in \seq{M-1}$, we have
\begin{align}
    R_m \leq \min\{4d \iota + 4\hat\beta_K \gamma^2d\iota +  2\hat\beta_K\sqrt{d \iota}\sqrt{S_m + 4R_m + 2R_{m+1} + KH\alpha^2},KH\}.\label{eq:rm_0}
\end{align}
For $R_{M-1}$, we have $R_{M-1} \leq KH$.
\end{lemma}
\begin{proof}
For $(k,h)$ where $I_h^k = 1$, by Lemma \ref{lemma:det} we have
\begin{align}
    \Big\|\hat \bSigma_{k,m}^{-1/2} \bphi_{k,h,m}\Big\|_2 \leq \Big\|\tilde\bSigma_{k,h,m}^{-1/2} \bphi_{k,h,m}\Big\|_2\cdot\sqrt{\frac{\det\hat \bSigma_{k,m}^{-1/2}}{\det \tilde\bSigma_{k,h,m}^{-1/2}}} \leq 2\Big\|\tilde\bSigma_{k,h,m}^{-1/2} \bphi_{k,h,m}\Big\|_2,\label{rm:00}
\end{align}
then substituting \eqref{rm:00} into \eqref{def:rm}, we have
\begin{align}
    R_m \leq 2\sum_{k=1}^K \sum_{h=1}^H\min\{1, I_h^k\hat\beta_k\Big\|\tilde\bSigma_{k,h,m}^{-1/2} \bphi_{k,h,m}\Big\|_2\}.\label{rm:01}
\end{align}
\eqref{rm:01} can be bounded by Lemma \ref{lemma:keysum}, with $\beta_{k,h} = I_h^k\hat\beta_k$, $\bar\sigma_{k,h} = \bar\sigma_{k,h,m}$, $\ab_{k,h} = \bphi_{k,h,m}$ and $\hat\bSigma_{k,h} = \tilde\bSigma_{k,h,m}$. We have
\begin{align}
    &\sum_{k=1}^K \sum_{h=1}^H\min\{1, I_h^k\hat\beta_k\Big\|\tilde\bSigma_{k,h,m}^{-1/2} \bphi_{k,h,m}\Big\|_2\}\notag \\
    & \leq 2d \iota + 2\hat\beta_K \gamma^2d\iota +  2\hat\beta_K\sqrt{\sum_{k=1}^K\sum_{h=1}^HI_h^k\big[[\bar\var_{k,m}\vvalue_{k, h+1}^{2^m}](s_h^k, a_h^k) + \error_{k,h,m}\big] + KH\alpha^2}\sqrt{d \iota}\notag \\
    & \leq 2d \iota + 2\hat\beta_K \gamma^2d\iota +  2\hat\beta_K\sqrt{\sum_{k=1}^K\sum_{h=1}^HI_h^k\big[[\var\vvalue_{k, h+1}^{2^m}](s_h^k, a_h^k) + 2\error_{k,h,m}\big] + KH\alpha^2}\sqrt{d \iota}\notag
\end{align}
Note that 
\begin{align}
\sum_{k=1}^K\sum_{h=1}^HI_h^kE_{k,h,m}
&= \sum_{k=1}^K\sum_{h=1}^HI_h^k\min\Big\{1,2\hat\beta_k\Big\|\hat\bSigma_{k,m}^{-1/2}\bphi_{k,h,m}\Big\|_2\Big\}\notag \\
&\quad +\sum_{k=1}^K\sum_{h=1}^H I_h^k\min\Big\{1, \hat\beta_k\Big\|\hat\bSigma_{k,m+1}^{-1/2}\bphi_{k,h,m+1}\Big\|_2\Big\}\notag \\
& \leq 2R_m + R_{m+1},\notag
\end{align}
by the definition of $R_m$ \eqref{def:rm}. This completes our proof.
\end{proof}

\begin{lemma}[Lemma 25, \citealt{zhang2021improved}]\label{lemma:sm}
Let $\{S_m, A_m\}_{m \in \seq{M}}$ be defined in \eqref{def:am} and \eqref{def:sm}, $G$ be defined in \eqref{def:g}. Then on the event $\event_{\ref{thm:concentrate:finite}}$, for $m \in \seq{M-1}$, we have 
\begin{align}
    S_m \leq |A_{m+1}| + G + 2^{m+1}(K+ 2R_0).\label{eq:sm_0}
\end{align}
\end{lemma}
\begin{proof}
The proof follows the proof of Lemma 25 in \citealt{zhang2021improved}. We have 
\begin{align}
    S_m & = \sum_{k=1}^K\sum_{h=1}^HI_h^k[\PP\vvalue_{k, h+1}^{2^{m+1}}](s_h^k, a_h^k) - ([\PP\vvalue_{k, h+1}^{2^{m}}](s_h^k, a_h^k))^2]\notag \\
    & = \sum_{k=1}^K\sum_{h=1}^HI_h^k[\PP\vvalue_{k, h+1}^{2^{m+1}}](s_h^k, a_h^k) - \vvalue_{k, h+1}^{2^{m+1}}(s_{h+1}^k)] + I_h^k[\vvalue_{k, h}^{2^{m+1}}(s_h^k) - ([\PP\vvalue_{k, h+1}^{2^{m}}](s_h^k, a_h^k))^2]\notag \\
    &\quad + \sum_{k=1}^K\sum_{h=1}^HI_h^k(\vvalue_{k, h+1}^{2^{m+1}}(s_{h+1}^k) - \vvalue_{k, h}^{2^{m+1}}(s_h^k))\notag \\
    & \leq A_{m+1} + \sum_{k=1}^K\sum_{h=1}^HI_h^k[\vvalue_{k, h}^{2^{m+1}}(s_h^k) - ([\PP\vvalue_{k, h+1}^{2^{m}}](s_h^k, a_h^k))^2] + \sum_{k=1}^KI_{h_k}^k \vvalue_{k, h_k+1}^{2^{m+1}}(s_{h_k+1}^k),\notag
\end{align}
where $h_k$ is the largest index such that $I_h^k = 1$. Note that if $h_k<H$, we have $I_{h_k}^k \vvalue_{k, h_k+1}^{2^{m+1}}(s_{h_k+1}^k) \leq 1 = 1-I_H^k$, and if $h_k = H$, we have $I_{h_k}^k \vvalue_{k, h_k+1}^{2^{m+1}}(s_{h_k+1}^k) = 0 = 1-I_H^k$. Thus, for both cases, we have
\begin{align}
    S_m \leq A_{m+1} + \sum_{k=1}^K(1-I_H^k) + \sum_{k=1}^K\sum_{h=1}^HI_h^k[\vvalue_{k, h}^{2^{m+1}}(s_h^k) - ([\PP\vvalue_{k, h+1}^{2^{m}}](s_h^k, a_h^k))^2].\label{oppo:1}
\end{align}
For the third term in \eqref{oppo:1}, we have
\begin{align}
    &\sum_{k=1}^K\sum_{h=1}^HI_h^k[\vvalue_{k, h}^{2^{m+1}}(s_h^k) - ([\PP\vvalue_{k, h+1}^{2^{m}}](s_h^k, a_h^k))^2]\notag \\
    & \leq \sum_{k=1}^K\sum_{h=1}^HI_h^k[\vvalue_{k, h}^{2^{m+1}}(s_h^k) - ([\PP\vvalue_{k, h+1}](s_h^k, a_h^k))^{2^{m+1}}]\notag \\
    & = \sum_{k=1}^K\sum_{h=1}^H I_h^k (\vvalue_{k, h}(s_h^k) -[\PP\vvalue_{k, h+1}](s_h^k, a_h^k))\prod_{i=0}^m(\vvalue_{k, h}^{2^i}(s_h^k) + [\PP\vvalue_{k, h+1}](s_h^k, a_h^k)^{2^i})\notag \\
    & \leq 2^{m+1}\sum_{k=1}^K\sum_{h=1}^H I_h^k \max\{\vvalue_{k, h}(s_h^k) -[\PP\vvalue_{k, h+1}](s_h^k, a_h^k),0 \}\notag \\
    & \leq 2^{m+1}\sum_{k=1}^K\sum_{h=1}^H I_h^k \big[r(s_h^k, a_h^k) + 2\min\{1, \hat\beta_k\Big\|\hat \bSigma_{k,0}^{-1/2} \bphi_{k,h,0}\Big\|_2\}\big]\notag \\
    & \leq 2^{m+1}(K+ 2R_0),\label{oppo:2}
\end{align}
where the first inequality holds by recursively applying $\EE X^2 \geq (\EE X)^2$, the second one holds since $V_{k,h} \in [0,1]$, the third one holds due to Lemma \ref{lemma:expectation}, the last one holds due to the definition of $R_0$ in \eqref{def:rm}. Substituting \eqref{oppo:2} into \eqref{oppo:1} completes our proof. 
\end{proof}

\begin{lemma}[Lemma 25, \citealt{zhang2021improved}]\label{lemma:am}
Let $\{A_m, S_m\}_{m \in \seq{M}}$ be defined in \eqref{def:am} and \eqref{def:sm}. Then
we have $\PP(\event_{\ref{lemma:am}})>1-M\delta$, where
\begin{align}
    \event_{\ref{lemma:am}}: = \big\{\forall m \in \seq{M},\ |A_m| \leq \min\{\sqrt{2\zeta S_m} + \zeta, KH\}\big\}.\label{eq:am_0}
\end{align}
\end{lemma}
\begin{proof}
The proof follows the proof of Lemma 25 in \citet{zhang2021improved}. We use Lemma \ref{lemma:zhangvar} for each $m$. Let $x_{k,h} = I_h^k[[\PP V_{k,h+1}^{2^m}](s_h^k, a_h^k) - V_{k,h+1}^{2^m}(s_{h+1}^k)]$, then we have $\EE[x_{k,h}|\cF_{k,h}] = 0$ and $\EE[x_{k,h}^2|\cF_{k,h}] = I_h^k[\var\vvalue_{k, h+1}^{2^m}](s_h^k, a_h^k)$. Therefore, for each $m \in \seq{M}$, with probability at least $1-\delta$, we have
\begin{align}
A_m &= \sum_{k=1}^K \sum_{h=1}^H x_{k,h}  \leq \sqrt{2\zeta\sum_{k=1}^K \sum_{h=1}^HI_h^k[\var\vvalue_{k, h+1}^{2^m}](s_h^k, a_h^k)} + \zeta.\notag
\end{align}
Taking union bound over $m \in \seq{M}$ completes the proof. 
\end{proof}

\begin{lemma}\label{lemma:g}
Let $G$ be defined in \eqref{def:g}. Then we have $G \leq Md\iota/2$.
\end{lemma}
\begin{proof}
Recall the definition of $I_H^k$, by the fact that $\det(\hat\bSigma_{k+1,m}^{-1/2})<\det(\tilde\bSigma_{k,H,m}^{-1/2})$, we have
\begin{align}
    (1-I_H^k) = 1 &\Leftrightarrow \exists m \in \seq{M},\ \det(\hat\bSigma_{k,m}^{-1/2})/\det(\tilde\bSigma_{k,H,m}^{-1/2}) > 4\notag \\
    & \Rightarrow \exists m \in \seq{M},\ \det(\hat\bSigma_{k,m}^{-1/2})/\det(\hat\bSigma_{k+1,m}^{-1/2})>4.\notag
\end{align}
Let $\cD_m$ denote the indices $k$ such that 
\begin{align}
    \cD_m:=\bigg\{k \in [K]: \det(\hat\bSigma_{k+1,m})/\det(\hat\bSigma_{k,m}) > 16 \bigg\}.\notag
\end{align}
Then we have $G \leq |\cup_{m=0}^{M-1} \cD_m| \leq \sum_{m=0}^{M-1}|\cD_m|$. For each $m$, we have
\begin{small}
\begin{align}
     2|\cD_m|  <\sum_{k \in \cD_m}\log 16<\sum_{k \in \cD_m} \log(\det(\hat\bSigma_{k+1,m})/\det(\hat\bSigma_{k,m})) \leq \sum_{k=1}^K\log(\det(\hat\bSigma_{k+1,m})/\det(\hat\bSigma_{k,m})).\notag
\end{align}
\end{small}
Furthermore, by the facts that $\log\det(\hat\bSigma_{K+1,m}) \leq d \log[\text{tr}(\hat\bSigma_{K+1,m})/d]$ and $\text{tr}(\hat\bSigma_{K+1,m}) \leq \text{tr}(\lambda\Ib) + \sum_{k,h}\|\bphi_{k,h,m}\|_2^2/\bar\sigma_{k,h,m}^2 \leq d\lambda + KH/\alpha^2$, we have
\begin{align}
    \sum_{k=1}^K\log(\det(\hat\bSigma_{k+1,m})/\det(\hat\bSigma_{k,m})) &= \log(\det(\hat\bSigma_{K+1,m})/\det(\hat\bSigma_{1,m})) \notag \\
    & \leq d(\log(\lambda + KH/d\alpha^2) - \log \lambda).\notag
\end{align}
Therefore, $|\cD_m|$ can be upper bounded by
\begin{align}
    |\cD_m| \leq d/2\cdot \log(1+KH/(d\lambda\alpha^2)) = d/2\cdot \iota.\notag
\end{align}
Taking the summation over $m$ gives the upper bound of $G$. 
\end{proof}

Finally, we define the event $\event_{\ref{event:azuma}}$ as
\begin{align}
    \event_{\ref{event:azuma}}:=\bigg\{\sum_{k=1}^K\bigg(\sum_{h=1}^H r(s_h^k, a_h^k) - \vvalue_{1}^{\pi^k}(s_1^k)\bigg) \leq \sqrt{2K\log(1/\delta)}\bigg\}.\label{event:azuma}
\end{align}
By Azuma-Hoeffding inequality (Lemma \ref{lemma:azuma}) we have
\begin{align}
    \PP(\event_{\ref{event:azuma}})>1-\delta.\label{event:azuma:pp}
\end{align}

With all above lemmas, we are ready to prove Theorem \ref{thm:regret:finite}.

\begin{proof}[Proof of Theorem \ref{thm:regret:finite}]
All the following proofs are conditioned on $\event_{\ref{thm:concentrate:finite}}\cap \event_{\ref{lemma:am}}\cap \event_{\ref{event:azuma}}$, which happens with probability at least $1-(2M+1)\delta$ by union bound and the probabilities of individual events $\event_{\ref{thm:concentrate:finite}}, \event_{\ref{lemma:am}}, \event_{\ref{event:azuma}}$ specified in Lemma \ref{thm:concentrate:finite}, Lemma \ref{lemma:am} and \eqref{event:azuma:pp}. First we have $\text{Regret}(K) \leq\sum_{k=1}^K[\vvalue_{k,1}(s_1^k) - \vvalue_{1}^{\pi^k}(s_1^k)]$ by Lemma \ref{lemma:upper:finite}. Next we have
\begin{align}
    &\sum_{k=1}^K\vvalue_{k,1}(s_1^k)\notag \\
    &= \sum_{k=1}^K\sum_{h=1}^H\big[I_h^k[\vvalue_{k,h}(s_h^k) - \vvalue_{k,h+1}(s_{h+1}^k)] + (1-I_h^k)[\vvalue_{k,h}(s_h^k) - \vvalue_{k,h+1}(s_{h+1}^k)]\big]\notag \\
    & = \sum_{k=1}^K\Big[\sum_{h=1}^HI_h^kr(s_h^k, a_h^k) +\sum_{h=1}^HI_h^k\big[V_{k,h}(s_h^k) - r(s_h^k, a_h^k) - [\PP V_{k,h+1}](s_h^k, a_h^k)\big]\notag \\
    &\quad + \sum_{h=1}^HI_h^k[\PP V_{k,h+1}](s_h^k, a_h^k) - V_{k,h+1}(s_{h+1}^k)\Big] + \sum_{k=1}^K\sum_{h=1}^H(1-I_h^k)[\vvalue_{k,h}(s_h^k) - \vvalue_{k,h+1}(s_{h+1}^k)]\notag \\
    & \leq \sum_{k=1}^K\Big[\sum_{h=1}^Hr(s_h^k, a_h^k) +\sum_{h=1}^HI_h^k\big[V_{k,h}(s_h^k) - r(s_h^k, a_h^k) - [\PP V_{k,h+1}](s_h^k, a_h^k)\big]\notag \\
    &\quad + \sum_{h=1}^HI_h^k[\PP V_{k,h+1}](s_h^k, a_h^k) - V_{k,h+1}(s_{h+1}^k)\Big] + \sum_{k=1}^K(1-I_{h_k}^k)\vvalue_{k,h_k}(s_{h_k}^k)\notag,
\end{align}
where $h_k$ is the smallest index such that $I_{h_k}^k = 0$. Then we have
\begin{align}
    &\text{Regret}(K)\notag \\
    & \leq \underbrace{\sum_{k=1}^K\bigg(\sum_{h=1}^H r(s_h^k, a_h^k) - \vvalue_{1}^{\pi^k}(s_1^k)\bigg)}_{I_1} + \underbrace{\sum_{k=1}^K \sum_{h=1}^H I_h^k\big[V_{k,h}(s_h^k) - r(s_h^k, a_h^k) - [\PP V_{k,h+1}](s_h^k, a_h^k)\big]}_{I_2} \notag \\
    &\quad + \underbrace{\sum_{k=1}^K \sum_{h=1}^H I_h^k[[\PP V_{k,h+1}](s_h^k, a_h^k) - V_{k,h+1}(s_{h+1}^k)]}_{A_0} + \underbrace{\sum_{k=1}^K(1-I_H^k)}_{G}\notag \\
    & \leq \sqrt{2K\log(1/\delta)} + 2R_0 + |A_0| + G .\label{oop:1} 
\end{align}
The bound of $I_1$ comes from $\event_{\ref{event:azuma}}$. For $I_2$, by Lemma \ref{lemma:expectation} we have 
\begin{align}
    I_2 \leq 2\sum_{k=1}^K \sum_{h=1}^H I_h^k\min\{1, \hat\beta_k\Big\|\hat \bSigma_{k,0}^{-1/2} \bphi_{k,h,0}\Big\|_2\} = 2R_0.\notag
\end{align}
Next we bound $2R_0 + |A_0|$ in \eqref{oop:1}. Substituting \eqref{eq:sm_0} in Lemma \ref{lemma:sm} into \eqref{eq:am_0} in Lemma \ref{lemma:am}, we have
\begin{align}
    |A_m| &\leq \sqrt{2\zeta(|A_{m+1}| + G + 2^{m+1}(K+ 2R_0))} + \zeta \notag \\
    &\leq 
    \sqrt{2\zeta}\sqrt{|A_{m+1}| + 2^{m+1}(K+ 2R_0)} + \sqrt{2\zeta G} + \zeta
    \label{ooo:1}
\end{align}
Substituting \eqref{eq:sm_0} in Lemma \ref{lemma:sm} into \eqref{eq:rm_0} in Lemma \ref{lemma:rm}, we have
\begin{align}
    R_m &\leq 4d \iota + 4\hat\beta_K \gamma^2d\iota +  2\hat\beta_K\sqrt{d \iota}\sqrt{|A_{m+1}| + G + 2^{m+1}(K+ 2R_0) + 4R_m + 2R_{m+1} + KH\alpha^2}\notag \\
    & \leq 2\hat\beta_K\sqrt{d \iota}\sqrt{|A_{m+1}|+ 2^{m+1}(K+ 2R_0) + 4R_m + 2R_{m+1}} + \notag \\
    &\quad + \underbrace{4d \iota + 4\hat\beta_K \gamma^2d\iota + 2\hat\beta_K\sqrt{d \iota}\sqrt{G+KH\alpha^2}}_{I_c} 
    \label{ooo:2}
\end{align}
Calculating \eqref{ooo:1} + 2$\times$\eqref{ooo:2} and using $\sqrt{a} + \sqrt{b} \leq \sqrt{2(a+b)}$, we have
\begin{align}
    &|A_m| + 2R_m \notag \\
    &\leq 2I_c + \sqrt{2}\max\{2\hat\beta_K\sqrt{d \iota}, \sqrt{2\zeta}\}\notag \\
    & \quad \cdot\sqrt{2|A_{m+1}|+ 2\cdot2^{m+1}(K+ 2R_0) + 8R_m + 4R_{m+1} + |A_{m+1}| + 2^{m+1}(K+ 2R_0)}\notag \\
    & \leq 2I_c + \sqrt{8}\max\{2\hat\beta_K\sqrt{d \iota}, \sqrt{2\zeta}\}\notag \\
    & \quad \cdot\sqrt{|A_m| + 2R_m + |A_{m+1}|+ 2R_{m+1} + 2^{m+1}(K+ 2R_0 + |A_0|)}\notag 
\end{align}
Then by Lemma \ref{lemma:seq} with  $a_m = |A_m| + 2R_m \leq 3KH$ and $M = \log(3KH)/\log 2$,  $|A_0| + 2R_0$ can be bounded as
\begin{align}
    &|A_0| + 2R_0 \notag \\
    &\leq 22\cdot 8\max\{4\hat\beta_K^2d \iota, 2\zeta\} + 6 \cdot 2I_c + 4\cdot \sqrt{8}\max\{2\hat\beta_K\sqrt{d \iota}, \sqrt{2\zeta}\}\sqrt{2(K+|A_0|+2R_0)}\notag \\
    & \leq 352\max\{2\hat\beta_K^2d \iota, \zeta\} + 12(4d \iota + 4\hat\beta_K \gamma^2d\iota + 2\hat\beta_K\sqrt{d \iota}\sqrt{G+KH\alpha^2}) \notag \\
    &\quad + 16\max\{2\hat\beta_K\sqrt{d \iota}, \sqrt{2\zeta}\}\sqrt{K}+ 16\max\{2\hat\beta_K\sqrt{d \iota}, \sqrt{2\zeta}\}\sqrt{|A_0| + 2R_0}.\label{ppp:1}
\end{align}
By the fact $x \leq a\sqrt{x} + b \Rightarrow x \leq 2a^2 + 2b$, \eqref{ppp:1} implies
\begin{align}
    |A_0| + 2R_0 &\leq 1728\max\{2\hat\beta_K^2d \iota, \zeta\}+ 24(4d \iota + 4\hat\beta_K \gamma^2d\iota + 2\hat\beta_K\sqrt{d \iota}\sqrt{G+KH\alpha^2}) \notag \\
    &\quad + 32\max\{2\hat\beta_K\sqrt{d \iota}, \sqrt{2\zeta}\}\sqrt{K}.\label{ppp:2}
\end{align}
Finally, substituting \eqref{ppp:2} into \eqref{oop:1} and bounding $G$ by Lemma \ref{lemma:g}, the regret is bounded as
\begin{align}
    \text{Regret}(K) &\leq  1728\max\{2\hat\beta_K^2d \iota, \zeta\}+ 48(2d \iota + 2\hat\beta_K \gamma^2d\iota + \hat\beta_K\sqrt{d \iota}\sqrt{Md\iota/2+KH\alpha^2})\notag \\
    &\quad + Md\iota/2 +\big[\sqrt{2\log(1/\delta)}+ 32\max\{2\hat\beta_K\sqrt{d \iota}, \sqrt{2\zeta}\}\big]\sqrt{K},\notag 
\end{align}
which completes our proof.
\end{proof}

\subsection{Proof of Theorem \ref{prop:lowerbound}}
We have the following lemma to lower bound the regret for linear bandits. 
\begin{lemma}[Lemma 25, \citealt{zhou2021nearly}]\label{lemma:banditlowerbound}
Fix a positive real $0<\delta\le 1/3$,  and positive integers $K,d$ and assume that $K \geq d^2/(2\delta)$. 
Let $\Delta = \sqrt{\delta/K}/(4\sqrt{2})$ and
consider the linear bandit problems $\cL_{\bmu}$ parameterized with a parameter vector $\bmu\in \{-\Delta,\Delta\}^d$
and action set $\cA = \{-1, 1\}^d$ so that the reward distribution 
for taking action $\ba\in \cA$ is a Bernoulli distribution $B(\delta + \la \bmu^*,\ab\ra)$.
Then for any bandit algorithm $\cB$, there exists a $\bmu^* \in \{-\Delta, \Delta\}^d$ such that the expected pseudo-regret of $\cB$ over first $K$ steps on bandit $\cL_{\bmu^*}$ is lower bounded as follows: 
\begin{align}
    \EE_{\bmu^*}\text{Regret}(K) \geq \frac{d\sqrt{K\delta}}{8\sqrt{2}}.\notag
\end{align}
\end{lemma}
\begin{proof}[Proof of Theorem \ref{prop:lowerbound}]
The linear mixture MDP instance is similar to the MDP instances considered in \citet{zhou2020provably, zhou2021nearly, zhang2021reward}. The state space $\cS = \{\state_1, \state_2, \state_3\}$. The action space $\cA = \{\ab\} = \{-1, +1\}^{d-1}$. The reward function satisfies $\reward(\state_1, \ab) = \reward(\state_2, \ab) = 0$ and $\reward(\state_3, \ab) = 1/H$. The transition probability satisfies $\PP(\state_2|\state_1, \ab) = 1-(\delta+\ip{\bmu,\ba})$ and $\PP(\state_3|\state_1,\ab) = \delta+\ip{\bmu,\ba}$, where $\delta=1/6$ and $\bmu\in \{-\Delta,\Delta\}^{d-1}$ with $\Delta = \sqrt{\delta/K}/(4\sqrt{2})$. 

First, similar to the proof in (Section E.1, \citealt{zhou2021nearly}), we can verify that when $K \geq (d-1)/(192(\pnorm - 1))$, our instance is a $\pnorm$-bounded linear mixture MDP with $\PP(s'|s,\ab) = \la \bphi(s'|s,\ab), \btheta\ra$, where
\begin{align*}
 &\bphi(s'|s,\ab) = 
\begin{cases}
     (\alpha(1-\delta), -\beta\ab^\top)^\top, &s = \state_1, s' = \state_{2}\,;\\
   (\alpha\delta, \beta\ab^\top)^\top, &s = \state_1, s' = \state_{3}\,;\\
    (\alpha,\zero^\top)^\top, &s \in \{ \state_{2}, \state_{3}\}, s' = s\,; \\
    \zero, &\text{otherwise}\,.
    \end{cases},\notag \\
    &\btheta = (1/\alpha, \bmu^\top/\beta)^\top.
\end{align*}
where $\alpha = \sqrt{1/(1+\Delta(d-1))},\ \beta =\sqrt{\Delta/(1+\Delta(d-1))}$. 

Second, our instance can be regarded as a linear bandit instance with a Bernoulli reward distribution $B(\delta + \la \btheta,\ab\ra)$.  Therefore, the lower bound of regret for linear mixture MDP directly follows the regret for linear bandits in Lemma \ref{lemma:banditlowerbound}, by picking $\bmu = \bmu^*$ and $\delta = 1/6$.  
\end{proof}

\section{Auxiliary Lemmas}
\begin{lemma}[Azuma-Hoeffding inequality, \citealt{azuma1967weighted}]\label{lemma:azuma} 
Let $M>0$ be a constant. Let $\{x_i\}_{i=1}^n$ be a stochastic process, $\cG_i = \sigma(x_1,\dots, x_i)$ be the $\sigma$-algebra of $x_1,\dots, x_i$. Suppose $\EE[x_i|\cG_{i-1}]=0$, $|x_i| \leq M$ almost surely. Then, for any $0<\delta<1$, we have 
\begin{align}
    \PP\bigg(\sum_{i=1}^n x_i\leq M\sqrt{2n \log (1/\delta)}\bigg)>1-\delta.\notag
\end{align} 
\end{lemma}

\begin{lemma}[Lemma 11, \citealt{zhang2021model}]
\label{lemma:zhangvar}
Let $M>0$ be a constant. Let $\{x_i\}_{i=1}^n$ be a stochastic process, $\cG_i = \sigma(x_1,\dots, x_i)$ be the $\sigma$-algebra of $x_1,\dots, x_i$. Suppose $\EE[x_i|\cG_{i-1}]=0$,
 $|x_i| \leq M$ and $\EE[x_i^2|\cG_{i-1}]<\infty$ almost surely. Then, for any $\delta,\epsilon >0$, we have
\begin{align}
    &\PP\bigg(\bigg|\sum_{i=1}^n x_i\bigg| \leq 2\sqrt{2\log(1/\delta)\sum_{i=1}^n\EE[x_i^2|\cG_{i-1}]} + 2\sqrt{\log(1/\delta)}\epsilon + 2M\log(1/\delta)\bigg)\notag\\
    &\quad >1-2(\log(M^2n/\epsilon^2)+1)\delta.\notag
\end{align}
\end{lemma}

\begin{lemma}[Unbounded Freedman's inequality, \citealt{dzhaparidze2001bernstein, fan2017martingale}]\label{lemma:dz}
Let $\{x_i\}_{i=1}^n$ be a stochastic process, $\cG_i = \sigma(x_1,\dots, x_i)$ be the $\sigma$-algebra of $x_1,\dots, x_i$. Suppose $\EE[x_i|\cG_{i-1}]=0$ and $\EE[x_i^2|\cG_{i-1}]<\infty$ almost surely. Then, for any $a,v,y>0$, we have
\begin{align}
    \PP\bigg(\sum_{i=1}^n x_i>a,\ \sum_{i=1}^n\big(\EE[x_i^2|\cG_{i-1}] + x_i^2\ind\{|x_i|>y\}\big)<v^2\bigg)\leq \exp\bigg(\frac{-a^2}{2(v^2 + ay/3)}\bigg).\notag
\end{align} 
\end{lemma}

\begin{lemma}[Lemma 8, \citealt{zhang2021improved}]\label{lemma:zhangconcen}
Let $\{x_i \geq 0\}_{i\geq 1}$ be a stochastic process, $\{\cG_i\}_{i \geq 1}$ be a filtration satisfying $x_i$ is $\cG_i$-measurable. We also have $|x_i| \leq 1$. Then for any $c \geq 1$ we have
\begin{align}
    \PP\bigg(\exists n,\ \sum_{i=1}^n x_i \geq 4c\log(4/\delta),\ \sum_{i=1}^n \EE[x_i|\cG_{i-1}] \leq c\log(4/\delta)\bigg) \leq \delta.\notag
\end{align}
\end{lemma}

\begin{lemma}[Lemma 11, \citealt{abbasi2011improved}]\label{lemma:sumcontext}
For any $\lambda>0$ and sequence $\{\xb_k\}_{k=1}^K \subset \RR^d$
for $k\in [K]$, define $\Zb_k = \lambda \Ib+ \sum_{i=1}^{k-1}\xb_i\xb_i^\top$.
Then, provided that $\|\xb_k\|_2 \leq L$ holds for all $k\in [K]$,
we have
\begin{align}
    \sum_{k=1}^K \min\{1, \|\xb_k\|_{\Zb_{k}^{-1}}^2\} \leq 2d\log(1+KL^2/(d\lambda)).\notag
\end{align}
\end{lemma}

\begin{lemma}[Lemma 12, \citealt{abbasi2011improved}]\label{lemma:det}
Suppose $\Ab, \Bb\in \RR^{d \times d}$ are two positive definite matrices satisfying $\Ab \succeq \Bb$, then for any $\xb \in \RR^d$, $\|\xb\|_{\Ab} \leq \|\xb\|_{\Bb}\cdot \sqrt{\det(\Ab)/\det(\Bb)}$.
\end{lemma}

\begin{lemma}[Lemma 12, \citealt{zhang2021improved}]\label{lemma:seq}
Let $\lambda_1, \lambda_2,\lambda_4>0$, $\lambda_3 \geq 1$ and $\kappa =\max\{\log_2 \lambda_1, 1\}$. Let $a_1,\dots, a_\kappa$ be non-negative real numbers such that $a_i \leq \min\{\lambda_1, \lambda_2\sqrt{a_i + a_{i+1} + 2^{i+1}\lambda_3} + \lambda_4\}$ for any $1 \leq i \leq \kappa$. Let $a_{\kappa+1} = \lambda_1$. Then we have $a_1 \leq 22\lambda_2^2 + 6\lambda_4 + 4\lambda_2\sqrt{2\lambda_3}$.
\end{lemma}

\bibliographystyle{ims}
\bibliography{reference}
\end{document}